%% file: main.tex
\documentclass[sigconf]{aamas}

\usepackage{balance} 

\doi{TFHZ1924}

\makeatletter
\gdef\@copyrightpermission{
  \begin{minipage}{0.2\columnwidth}
   \href{https://creativecommons.org/licenses/by/4.0/}{\includegraphics[width=0.90\textwidth]{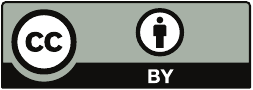}}
  \end{minipage}\hfill
  \begin{minipage}{0.8\columnwidth}
   \href{https://creativecommons.org/licenses/by/4.0/}{This work is licensed under a Creative Commons Attribution International 4.0 License.}
  \end{minipage}
  \vspace{5pt}
}
\makeatother

\setcopyright{ifaamas}
\acmConference[AAMAS '26]{Proc.\@ of the 25th International Conference
on Autonomous Agents and Multiagent Systems (AAMAS 2026)}{May 25 -- 29, 2026}
{Paphos, Cyprus}{C.~Amato, L.~Dennis, V.~Mascardi, J.~Thangarajah (eds.)}
\copyrightyear{2026}
\acmYear{2026}
\acmDOI{}
\acmPrice{}
\acmISBN{}

\title[AAMAS-2026 Formatting Instructions]{A Generic Framework for Fair Consensus Clustering in Streams}

\author{Diptarka Chakraborty}
\affiliation{
  \institution{National University of Singapore}
  \city{Singapore}
  \country{Singapore}}
\email{diptarka@nus.edu.sg}

\author{Kushagra Chatterjee}
\affiliation{
  \institution{National University of Singapore}
  \city{Singapore}
  \country{Singapore}}
\email{kushagra.chatterjee@u.nus.edu}

\author{Debarati Das}
\affiliation{
  \institution{Pennsylvania State University}
  \city{Pennsylvania}
  \country{USA}}
\email{debaratix710@gmail.com}

\author{Tien-Long Nguyen}
\affiliation{
  \institution{Pennsylvania State University}
  \city{Pennsylvania}
  \country{USA}}
\email{tfn5179@psu.edu}

\begin{abstract}
\emph{Consensus clustering} seeks to combine multiple clusterings of the same dataset, potentially derived by considering various non-sensitive attributes by different agents in a multi-agent environment, into a single partitioning that best reflects the overall structure of the underlying dataset. Recent work by Chakraborty et al. [COLT’25] introduced a \emph{fair} variant under proportionate fairness and obtained a constant-factor approximation by naively selecting the best \emph{closest fair} input clustering; however, their offline approach requires storing all input clusterings, which is prohibitively expensive for most large-scale applications.

In this paper, we initiate the study of fair consensus clustering in the streaming model, where input clusterings arrive sequentially and memory is limited. We design the first constant-factor algorithm that processes the stream while storing only a logarithmic number of inputs. 
En route, we introduce a new generic algorithmic framework that integrates \emph{closest fair clustering} with \emph{cluster fitting}, yielding improved approximation guarantees not only in the streaming setting but also when revisited offline.
Furthermore, the framework is fairness-agnostic: it applies to any fairness definition for which an approximately close fair clustering can be computed efficiently. Finally, we extend our methods to the more general \emph{$k$-median consensus clustering} problem.
\end{abstract}

\keywords{Consensus Clustering, Fairness, Approximation Algorithms, Streaming Algorithms, Combinatorial Optimization}

\hypersetup{
    linkcolor=blue
    ,citecolor=orange
    ,filecolor=purple
    ,urlcolor=orange
    ,menucolor=purple
    ,runcolor=purple
}

\usepackage[lined,linesnumbered,ruled,vlined]{algorithm2e}

\usepackage{amsmath, amsthm, amsfonts, graphicx, color, subcaption, enumerate, bm, array, mathtools}

\usepackage{xparse} 
\usepackage{enumitem} 

\makeatletter
\renewcommand{\nllabel}[1]
 {{\let\@currentlabel\algocf@currentlabel
  \let\@currentcounter\algocf@currentcounter
  \label{#1}}}%

\renewcommand{\algocf@nl@sethref}[1]{%
  \renewcommand{\theHAlgoLine}{\thealgocfproc.#1}%
  \hyper@refstepcounter{AlgoLine}%
  \gdef\algocf@currentlabel{#1}%
  \gdef\algocf@currentcounter{AlgoLine}%
 }%
\makeatother

\usepackage[capitalise,nameinlink]{cleveref}

\usepackage[colorinlistoftodos,prependcaption,textsize=tiny]{todonotes}

\usepackage{multicol, multirow}
\usepackage{thmtools, thm-restate}
\usepackage{ragged2e}
\usepackage{lineno}
\usepackage{framed}
\usepackage[framemethod=tikz]{mdframed}

\usepackage{subfiles}
\usepackage{dirtytalk}
\graphicspath{{graphics/}}

\crefname{claim}{Claim}{Claims}
\crefname{property}{Property}{Properties}

\input{newcommand}

\begin{document}


\pagestyle{fancy}
\fancyhead{}


\maketitle 


\input{introduction}
\input{Preliminaries}
\input{fair-consensus-clustering}

\input{A-faster-fair-consensus}

\input{k-median-consensus}

\input{streaming.tex}

\section{Conclusion and Future Work}
In this paper, we initiate the study of fair consensus clustering under the $k$-median objective in the streaming model and present the first constant-factor algorithm that operates in sublinear space. Relative to prior work on fair consensus clustering~\cite{chakraborty2025towards}, our contribution advances the state of the art by handling streaming data with sublinear memory and by delivering the first constant-factor guarantee for the more general $k$-median objective. We also introduce a new generic algorithmic framework that works irrespective of the specific fairness notion and the number of colored groups (whether disjoint or not) -- in fact, it works with any constraint -- provided the corresponding closest fair (constrained) clustering problem can be efficiently approximated.

Improving space usage, particularly for the $k$-median variant, and tightening approximation guarantees are among a few interesting open directions. An exciting avenue is to narrow the gap between the approximation factor achieved by our framework and the best attainable for the closest fair clustering problem. However, existing hardness results for fair consensus clustering, together with exact algorithms for closest fair clustering in certain special cases~\cite{chakraborty2025towards}, preclude any fully generic fair consensus algorithm from matching the approximation factor of the closest fair clustering exactly. Another intriguing direction of research is to explore learning-augmented approaches, leveraging machine-learned predictors further to improve the approximation quality in fair consensus clustering, at least empirically.

We would like to mention that in this paper we consider two variants of the streaming model, one in which all the pairs for a particular input clustering arrive together (in a stream), which we refer to as an insertion-only stream, and then a general model where pairs may appear in an arbitrary order, which we refer to as a generalized insertion-only stream. Our $1$-median algorithm works for this generalized model, while extension to the  $k$-median only works for the insertion-only model. Thus, designing a streaming algorithm for the $k$-median variant in the generalized insertion-only model would be an interesting open direction. It would also be intriguing to study other possible streaming models in the context of consensus clustering, and we leave this as a potential future direction.



\begin{acks}
Diptarka Chakraborty was supported in part by an MoE AcRF Tier 1 grant (T1 251RES2303) and a Google South \& South-East Asia Research Award. Kushagra Chatterjee was supported by an MoE AcRF Tier 1 grant (T1 251RES2303). Debarati Das was supported in part by NSF grant 2337832. 
\end{acks}



\bibliographystyle{ACM-Reference-Format} 
\bibliography{sample}

\clearpage
\newpage

\appendix

\input{appendix}
\input{improved-parameters}


\end{document}

%% file: newcommand.tex
\makeatletter
\g@addto@macro\bfseries{\boldmath}
\makeatother

\newtheorem{theorem}{Theorem}
\newtheorem{lemma}{Lemma}[section]
\newtheorem{claim}[lemma]{Claim}

\newtheorem{corollary}[lemma]{Corollary}
\newtheorem{observation}[lemma]{Observation}

\theoremstyle{definition}
\newtheorem{definition}[lemma]{Definition}

\theoremstyle{remark}

\newtheorem*{remark*}{Remark}

\def\polylog{\operatorname{polylog}}

\renewcommand{\epsilon}{\varepsilon}
\newcommand{\poly}{\operatorname{poly}}

\newcommand{\R}{\mathbb{R}}
\newcommand{\m}[1]{\mathcal{#1}}
\newcommand{\OO}{\widetilde{O}}

\newcommand{\n}{\nonumber}

\DeclareMathOperator{\dist}{dist}

\DeclareMathOperator*{\argmin}{arg\,min}
\renewcommand{\emptyset}{\varnothing}
\NewDocumentCommand{\expt}{m g}{%
  \IfNoValueTF{#2}
    {\mathbf{E}\left[#1\right]}%
    {\mathbf{E}_{#2}\left[ #1\right ]}%
} 

\newcommand{\fair}{\textit{Fair Clustering}}

\newcommand{\card}[1]{\left\lvert #1 \right\rvert}
\newcommand{\cls}[1]{\m{#1}}

\newcommand{\inpconclss}{\mathcal{I}}
\newcommand{\clusterings}{\cls{C}_{1},\cls{C}_{2}, \dots, \cls{C}_{m}} 
\NewDocumentCommand{\concls}{g}{%
  \IfNoValueTF{#1}
    {\cls{C}}%
    {\cls{C}_{#1}}%
} 
\NewDocumentCommand{\fairconcls}{g}{%
  \IfNoValueTF{#1}
    {\mathcal{F}}%
    {\mathcal{F}_{#1}}%
} 
\newcommand{\candidateSet}{\widetilde{\mathcal{F}}} 
\newcommand{\rcandidateSet}{\m{R}} 
\newcommand{\candidate}[1]{\widetilde{#1}} 
\newcommand{\tripleconcls}{T} 
\newcommand{\fairtriple}{\widetilde{T}} 
\NewDocumentCommand{\clsfitting}{g}{%
  \IfNoValueTF{#1}
    {\mathrm{ClusterFitting}}%
    {\mathrm{ClusterFitting}\left(#1\right)}%
} 

\newcommand{\outconcls}{\cls{F}} 
\newcommand{\obj}[1]{\mathrm{Obj}\left( #1 \right)} 
\newcommand{\approxFactorFairCorCls}{\eta}
\newcommand{\approxFactorClsFair}{\gamma}
\newcommand{\approxFactorCorCls}{\rho} 
\newcommand{\runtimeClsFair}[1]{t_{1}(#1)}
\newcommand{\runtimeFairCorCls}[1]{t_{2}(#1)}
\newcommand{\costcor}[1]{\operatorname{cost}(#1)} 
\newcommand{\optconval}{\mathrm{OPT}} 
\newcommand{\avgconval}{\mathrm{AVG}}
\newcommand{\optconcls}{\cls{F}^{*}} 
\newcommand{\unalignedSet}[1]{I_{#1}}
\newcommand{\optclsk}[1]{M_{#1}} 
\NewDocumentCommand{\fairmediank}{g}{%
  \IfNoValueTF{#1}
    {\cls{C}^{*}}%
    {\cls{C}^{*}_{#1}}%
} 

\NewDocumentCommand{\optobjk}{g}{%
  \IfNoValueTF{#1}
    {O}%
    {O_{#1}}%
} 

\newcommand{\coreset}{P} 
\newcommand{\errCoreset}{\epsilon} 
\newcommand{\impSet}{\m{X}} 
\newcommand{\metSpace}{\m{M}} 
\newcommand{\neigh}[1]{\mathscr{N}_{#1}} 
\newcommand{\farNeigh}[1]{\mathscr{F}_{#1}} 
\NewDocumentCommand{\sampSet}{g}{%
  \IfNoValueTF{#1}
    {\m{S}}%
    {\m{S}_{#1}}%
} 
\newcommand{\rate}{\tau} 
\newcommand{\sampRate}{p} 
\newcommand{\sampDist}{\ell} 
\newcommand{\sampDF}{\lambda} 
\newcommand{\optKFac}{\xi} 
\newcommand{\mfsSet}{F}
\newcommand{\mfskappa}{\kappa} 
\newcommand{\mfsrho}{\mu} 
\newcommand{\near}[1]{\mathrm{Near}(#1)} 
\newcommand{\midcl}[1]{\mathrm{Mid}(#1)} 
\newcommand{\middense}[1]{\mathrm{Mid}^{\mathrm{dense}}(#1)} 
\newcommand{\midsparse}[1]{\mathrm{Mid}^{\mathrm{sparse}}(#1)} 
\newcommand{\far}[1]{\mathrm{Far}(#1)} 

\newcommand{\infsmall}{\epsilon_{1}}

\newcommand{\BibTeX}{\rm B\kern-.05em{\sc i\kern-.025em b}\kern-.08em\TeX}

\newcommand{\objec}{\operatorname{Obj}}

\newcommand{\cset}{\mathcal{I}}
\newcommand{\oset}{\mathsf{Z}}
\newcommand{\canset}{\widetilde {\m{F}}}

\newcommand{\fc}{\texttt{find-candidates}}

\newcommand{\sone}{\texttt{st-1Med}}

\makeatletter
\newcommand{\setword}[2]{%
  \phantomsection
  #1\def\@currentlabel{\unexpanded{#1}}\label{#2}%
}
\makeatother

\newcommand{\loc}{\text{appendix}}

%% file: introduction.tex
\section{Introduction}
Clustering -- the task of partitioning data points into groups based on their mutual similarity or dissimilarity -- lies at the core of unsupervised learning and is pervasive across machine learning and data analysis applications. In many settings, each data point represents an individual endowed with protected attributes, which can be encoded by assigning a specific color to each point. While traditional clustering algorithms typically succeed in optimizing application-specific objectives, they often fall short in ensuring fairness in their outputs, risking the introduction or perpetuation of biases against marginalized groups delineated by sensitive attributes such as gender, ethnicity, and race~\cite{kay2015unequal, bolukbasi2016man}. These biases often stem not from the algorithms themselves, but rather from the historical marginalization inherent in the data used for training. Consequently, mitigating such biases to achieve fair outcomes has emerged as a central topic to guarantee demographic parity~\cite{dwork2012fairness} and equal opportunity~\cite{hardt2016equality}, not only in clustering but also in any machine learning driven process.

A seminal step toward ensuring fairness in several unsupervised learning tasks is the notion of fair clustering introduced by~\cite{chierichetti2017fair}, where points are colored and each partition/cluster must maintain a certain "balance" among various colored points. Intuitively, this balance promotes fair representation and helps address disparate impact. Motivated by this notion of fairness, recent work~\cite{chakraborty2025towards} has initiated the study of the fair variant of the \emph{consensus clustering} problem. Consensus clustering (also referred to as \emph{median partition}), a fundamental problem in machine learning and data analysis, seeks to aggregate a set of input clusterings -- often derived from distinct, non-sensitive attributes -- into a single clustering that best captures the collective structure of the data by minimizing the \emph{median objective}, i.e., the sum of distances to the inputs. Here, the distance between two clusterings is measured by counting the pairs of points co-clustered in one but not the other. This problem is particularly relevant in applications such as gene integration in bioinformatics~\cite{filkov2004integrating, filkov2004heterogeneous}, data mining~\cite{topchy2003combining}, and community detection~\cite{lancichinetti2012consensus}. However, the problem is \texttt{NP}-hard~\cite{kvrivanek1986np, swamy2004correlation} and even \texttt{APX}-hard -- precluding any $(1+\epsilon)$-approximation for arbitrary $\epsilon > 0 $ -- even with only three input clusterings~\cite{BonizzoniVDJ08}. The best-known approximation factor to date is $11/7$~\cite{ailon2008aggregating}, although several heuristics have been proposed that are known to work well in practice (e.g., \cite{goder2008consensus, monti2003consensus, wu2014k}). 

The standard consensus clustering procedures do not explicitly ensure fairness, whether it is \emph{proportional fairness} or \emph{statistical fairness}, in the aggregated clustering. For instance, consider community detection in social networks: valid partitions might arise from non-sensitive attributes, such as age groups or food preferences; yet, the final consensus partition should be fair, with each community reflecting an appropriate representation across protected attributes, like race or gender. The known standard consensus algorithms, whether being approximation algorithms or heuristics, fail to achieve such fairness. Addressing this gap, recent work~\cite{chakraborty2025towards} has initiated the incorporation of fairness constraints into consensus clustering and provided constant-factor approximation algorithms.

However, modern large-scale applications, ranging from federated learning systems to continuous data collection on online platforms, rarely permit full access to the entire input at once.
This motivates the study of space-efficient algorithms in the streaming model, where the input arrives sequentially and the algorithm processes points on the fly while retaining only a small subset or a compact \emph{sketch} of the data. At the end of the stream, all downstream computation is performed using this sketch. 
This naturally raises the question: can we maintain fairness efficiently in the streaming setting, achieving sublinear, ideally logarithmic, space per reported output, while preserving strong approximation guarantees?
Although multiple passes over the stream can sometimes be permitted, single-pass algorithms are preferred because additional passes are often cost-prohibitive in real-world applications (see~\cite{Muthukrishnan05, BabcockEtAl02}).

Further, in the insertion-only streaming model, which is the most commonly studied setting for clustering (e.g.,~\cite{CharikarOCallaghanPanigrahy03, rosman2014coresets, braverman2019streaming, schmidt2019fair}), points from an underlying metric space arrive one by one, and deletions are disallowed. In this paper, we initiate the study of fair consensus clustering in the streaming model.
While~\cite{chakraborty2025towards} provides constant-factor algorithms for fair consensus, their approach effectively computes a closest fair clustering for each input clustering, evaluates the median objective for each of these candidates, and returns the best -- necessitating explicit storage of all input clusterings and their closest fair clusterings and thus precluding sublinear-space streaming computation. 
It raises a natural question: can we obtain comparable approximation guarantees in a streaming model using only sublinear space?
We answer this in the affirmative: with a single pass over the input, we obtain constant-factor approximations for fair consensus clustering not only for the (1-)median objective, as in~\cite{chakraborty2025towards}, but also for the more general \emph{$k$-median} objective. Informally, $k$-median (fair) consensus clustering seeks $k$ (fair) representative partitions, generalizing the standard (fair) consensus clustering framework. As a natural extension, it has been applied to model selection and clustering ensemble summarization, image segmentation ensembles, bioinformatics, and community detection~\cite{topchy2005clustering, caruana2006meta, bellec2010multi, lancichinetti2012consensus}.

\subsection{Our Contribution}

\textbf{Fair Consensus Clustering.}
In this problem, we are given a collection of input clusterings 
$\mathcal{I} = \{\mathcal{C}_1, \mathcal{C}_2, \ldots, \mathcal{C}_m\}$ defined on a common ground set $V$ (where $|V| = n$), 
arriving sequentially in a stream, together with a fairness constraint $\mathcal{P}$. 
The objective is to compute a \emph{fair consensus clustering} $\mathcal{C}$ that minimizes the total distance 
between $\mathcal{C}$ and each input clustering in $\mathcal{I}$, while satisfying the fairness constraint 
$\mathcal{P}$ and using small space. Formally, we establish the following result:
\begin{itemize}
    \item Suppose there exists a $\gamma$-approximation algorithm for the closest fair clustering problem 
    under constraint $\mathcal{P}$, for some parameter $\gamma > 1$. 
    Then there exists a randomized polynomial time streaming algorithm that, given the input clusterings $\mathcal{I}$, 
    computes a fair clustering $\mathcal{F}$ in $ O(n\log( mn ))$ space, achieving a $(\gamma + 1.995)$-approximation with probability at least 
    $1 - 1/poly(m)$.
\end{itemize}

It is worth noting that since an intended output of the fair consensus clustering problem is a clustering on $V$, it is of size $\Omega(n)$. Furthermore, it follows from a simple counting argument that the number of fair clusterings, even in the equi-proportionate case, is $n^{\Omega(n)}$, and thus a standard information-theoretic (encoding-decoding based) argument shows that the space requirement must be $\Omega(n \log n)$ bits. So the space usage of our streaming algorithm is nearly optimal.

\textbf{$k$-Median Fair Consensus Clustering.}
Next, we consider a more general setting where, instead of finding a single representative clustering, 
the goal is to identify a collection of $k$ representative clusterings 
$\mathcal{S} = \{\mathcal{S}_1, \ldots, \mathcal{S}_k\}$ 
that together minimize the total distance between each input clustering and its nearest representative, 
while ensuring that each representative $\mathcal{S}_i$ satisfies the fairness constraint $\mathcal{P}$. Formally, we show the following:

\begin{itemize}
    \item Suppose there exists a $\gamma$-approximation algorithm for the closest fair clustering problem 
    under constraint $\mathcal{P}$, for some $\gamma > 1$. 
    Then there exists a randomized polynomial time streaming algorithm that, given the input clusterings $\mathcal{I}$, 
    computes a $k$-median fair consensus clustering $\mathcal{S}$ in 
    $O(k^{2}n\polylog(mn))$ space, achieving a $(1.0151\gamma + 1.99951)$-approximation with probability at least 
    $1 - 1/poly(m)$.
\end{itemize}

To the best of our knowledge, this is the first work to study \emph{fair consensus clustering} in the streaming model, and also the first to address its more general variant, \emph{$k$-median fair consensus clustering}. In fact, for the offline version of the \emph{$k$-median fair consensus clustering} problem, we obtain an improved approximation guarantee of $(\gamma + 1.92)$.

Building on the recent results of~\cite{chakraborty2025towards}, which study the closest fair clustering problem under two-color fairness constraints, 
as a corollary, we obtain the following approximation guarantees for the $k$-median fair consensus clustering problem in the streaming model: 
a $3.01461$-approximation when the population is equi-proportionate; 
a $19.25621$-approximation when the population ratio (of two colored groups) is $p\!:\!1$ for any positive integer $p > 1$; 
and a $35.49781$-approximation for the general case where the population ratio is 
$p\!:\!q$ where $p, q > 1$ are positive integers. 


We further remark that our proposed algorithms are \emph{robust to the specific definition of the fairness constraint}. 
This is because the fair component is used in a black-box manner: as long as there exists a good exact or approximation algorithm for the underlying closest fair clustering problem under a given fairness constraint, our streaming framework remains applicable. 
Consequently, the approach generalizes naturally to various settings, 
such as those involving an arbitrary number of colors (even with potentially overlapping colored groups), heterogeneous population ratios, 
or alternative notions like statistical fairness. In fact, our framework works for any constraint, not necessarily for fairness constraints, provided we have an algorithm to find an (approximately) closest constrained clustering (a nearest neighbor in the constrained set).

\textbf{Framework Overview.}
To design our algorithms, we introduce a \emph{new two-phase framework} for fair consensus clustering, adapting the frameworks of~\cite{chakraborty2023clustering, chakraborty2025improved}. 
In the first phase, we construct a candidate set of consensus clusterings by combining two ideas: \emph{cluster fitting} and \emph{closest fair clustering}, and then select the one minimizing the overall objective. 
The primary challenge is adapting this process to the streaming setting, where it is not possible to store all the input clusterings simultaneously. 
To overcome this, we use the power of uniform sampling 
that retains only a logarithmic number of input clusterings while still guaranteeing a good approximation to the true consensus. 
Moreover, we show that using an additional uniformly sampled subset of the input, again of only logarithmic size, is sufficient to reliably identify the best candidate consensus clusterings. This results in a sublinear-space algorithm that stores only a logarithmic number of inputs while maintaining strong approximation guarantees.


Further to ensure fairness, our approach treats any available closest fair clustering algorithm as a \emph{black box}, integrating it seamlessly into our consensus framework. 
This modularity ensures that our overall approximation inherits the fairness-approximation factor of the underlying subroutine. 
Consequently, whenever the closest fair clustering subroutine admits an exact or improved approximation (e.g., in the equi-proportionate two-color case due to~\cite{chakraborty2025towards}), our overall algorithm immediately yields a more efficient and tighter approximation.

We further extend this framework to the \emph{$k$-median fair consensus clustering} problem, where the goal is to output $k$ fair representative clusterings rather than one. 
Here, we develop a more carefully structured sampling method to ensure sufficient representation across all $k$ clusters, along with fairness enforcement for each representative. Furthermore, to identify the best $k$-representative clusterings, we employ a more sophisticated technique, \emph{monotone faraway sampling}~\cite{braverman2021metric} which, once again, requires only a logarithmic number of input samples.
The combination of these sampling and fairness components yields the first streaming algorithms providing constant-factor approximations for both 1-median and $k$-median fair consensus clustering.

\subsection{Other Related Works}
Since the concept of fair clustering was introduced in~\cite{chierichetti2017fair}, different variants of the clustering problems have been explored while imposing fairness constraints, including $k$-center/median/means clustering~\cite{chierichetti2017fair, HuangJV19}, scalable methods~\cite{BackursIOSVW19}, proportional clustering~\cite{ChenFLM19}, fair representational clustering~\cite{BeraCFN19, bercea2019cost}, pairwise fairness~\cite{bandyapadhyay2024coresets, bandyapadhyay2024polynomial, bandyapadhyay2025constant}, correlation clustering~\cite{pmlr-v108-ahmadian20a, ahmadi2020fair, ahmadian2023improved}, median rankings~\cite{wei22, chakraborty2022, chakraborty2025improved}, and consensus clustering~\cite{chakraborty2025towards, chakraborty2026}, among others.

Clustering problems have been well-explored in streaming models, mainly in the insertion-only streaming model. A central algorithmic paradigm in this model is to maintain \emph{coresets} or other succinct summaries that preserve the clustering structure, enabling approximate solutions for objectives such as $k$-median, $k$-means, and $k$-center with provable guarantees. Different coreset constructions are known even with fairness constraints~\cite{huang2019coresets, schmidt2019fair, chhaya2022coresets, braverman2022power, xiong2024fair}. Seminal work established streaming approaches for clustering data streams and introduced merge-and-reduce frameworks for constructing coresets with polylogarithmic space and update time~\cite{bentley1980decomposable, HarPeledMazumdar04, FeldmanLangberg11, braverman2019streaming}. Unfortunately, to the best of our knowledge, none of the above implies any sublinear algorithm for the fair consensus clustering problem with a non-trivial approximation guarantee.

%% file: Preliminaries.tex
\section{Preliminaries}\label{sec:preliminary}

In this section, we define key terms and concepts that are essential for understanding the proofs and algorithms presented.

\begin{definition}[$\fair$]\label{def:fair-clustering}
Given a set of points $V$ and a fairness constraint $\m{P}$, we call a clustering $\m{F}$ of $V$ a $\fair$ if every cluster $F \in \m{F}$ satisfies the constraint $\m{P}$.
\end{definition}

For example, Chierichetti et al.~\cite{chierichetti2017fair} considered the case where the set $V$ is partitioned into two disjoint groups (or colors), red and blue, denoted by $R$ and $B$, respectively. In this setting, the fairness constraint $\m{P}$ requires that each cluster $F \in \m{F}$ preserves the global color ratio, i.e.,
\[
    \frac{|F \cap R|}{|F \cap B|} = \frac{|R|}{|B|}.
\]
Equivalently, every cluster must contain the same proportion of red and blue points as in the entire dataset.

For two clustering $\m{C}$ and $\m{C}'$ of $V$ we define $\dist(\m{C}, \m{C'})$ as the distance between two clustering $\m{C}$ and $\m{C}'$. The distance is measured by the number of pairs $(u,v)$ that are together in $\m{C}$ but separated by $\m{C}'$ and the number of pairs $(u,v)$ that are separated by $\m{C}$ but together in $\m{C}'$. More specifically,
\begin{align*}
\dist(\m{C}, \m{C}') = 
\big| \big\{ \{u,v\} \mid\,& u,v \in V,\; 
[u \sim_{\m{C}} v \land u \not\sim_{\m{C}'} v] \\
&\lor [u \not\sim_{\m{C}} v \land u \sim_{\m{C}'} v] \big\} \big|
\end{align*}
where $u \sim_{\m{C}} v$ denotes whether both $u$ and $v$ belong to the same cluster in $\m{C}$ or not.

\begin{definition}[Closest $\fair$]
    Given an arbitrary clustering $\m{D}$, a clustering $\m{N}^*$ is called a closest $\fair$ to $\m{D}$ if for all $\fair$ $\m{N}$ we have $\dist(\m{D}, \m{N}) \geq \dist(\m{D}, \m{N}^*)$. 
\end{definition}

We denote a closest $\fair$ by $\m{N}^*$.

\paragraph{$\gamma$-close $\fair$} We call a $\fair$ $\m{N}$ a $\gamma$-close $\fair$ to a clustering $\m{D}$ if
\[
    \dist(\m{D}, \m{N}) \leq \gamma \dist(\m{D}, \m{N}^*).
\]

\begin{definition}[$1$-median Consensus Clustering problem]
    Given a set of clusterings $\m{I} = \{\m{C}_1, \m{C}_2, \ldots, \m{C}_m\}$, the goal is find a clustering $\m{C}$ such that it minimizes the total distance between each input clustering and $\m{C}$. Formally, we seek to minimize the following objective
    \[
        \obj{\m{I}, \m{C}} = \sum_{C_i \in \m{I}} \dist(C_i, \m{C})
    \]
\end{definition}

The above objective is called the \emph{median} objective. In addition to the objective, when we want the clustering to be fair, it is called the \emph{$1$- median fair consensus clustering problem.}

 We generalize this problem to \emph{$k$-median consensus clustering} problem.

\begin{definition}[$k$-median Consensus Clustering problem]
    Given a collection of input clusterings \(\cset = \{\mathcal{C}_1, \ldots, \mathcal{C}_m\}\), the objective is to construct a set of \(k\) representative clusterings  
\[
    \oset = \{\mathcal{Z}_1, \ldots, \mathcal{Z}_k\}
\]
that minimizes the total distance between each input clustering and its nearest representative. Formally, we seek to minimize
\[
    \objec(\cset, \oset)
    = \sum_{\mathcal{C}_i \in \cset} \min_{\mathcal{Z}_j \in \oset} \dist(\mathcal{C}_i, \mathcal{Z}_j).
\]
\end{definition}

In addition to the objective, when we want the clustering to be fair, we call it the \emph{$k$-median fair consensus clustering problem.}

\paragraph{($\beta, k$)-approximate Fair Consensus Clustering}: A set of representatives $\oset = \{\m{Z}_1, \ldots, \m{Z}_k\}$ such that each $\m{Z}_i$ is fair is called a $(\beta, k)$-approximate Fair Consensus Clustering if the following is true
    \begin{align*}
        \objec(\cset, \oset) \leq \beta \, \, \obj{\cset, \mathsf{Z}^*}
    \end{align*}
    where $\mathsf{Z}^* = \{\m{Z}_1^*, \ldots, \m{Z}_k^*\}$ be the set of optimal representatives.

Let us now define our streaming models.


\paragraph{Streaming Model.}
In this paper, we consider an \emph{insertion-only} streaming model.
The input is a sequence of triples $(p,j,b)$, where:
\begin{itemize}
  \item $p=(u,v)$ is an unordered pair of vertices in $V$,
  \item $j\in[m]$ identifies a clustering $\m{C}_j\in\m{I}$, and
  \item $b\in\{0,1\}$ indicates whether $u$ and $v$ lie in the same cluster in the clustering $\m{C}_j$ ($b=0$) or in different clusters ($b=1$).
\end{itemize}

Here, a triple $((u,v),j,b)$ refers an information about the clustering $\m{C}_j$. For each clustering $\m{C}_j$, we have $n^2$ such tuples, hence the length of the stream is $mn^2$.

In the insertion-only streaming framework commonly used in clustering, elements or points of a metric space are revealed sequentially. In consensus clustering, the metric space is the set of clusterings, so each clustering is treated as a single point in that space. Accordingly, following standard practice, we assume the stream delivers each clustering as a contiguous unit. We formalize this assumption by introducing a \emph{contiguity property} in our streaming model.

\noindent \textbf{Contiguity Property}: For every $j\in[m]$, all triples referring to clustering $\m{C}_j$
appear consecutively in the stream.
Formally, if
\[
  (p_\alpha,j_\alpha,b_\alpha),\;
  (p_\beta,j_\beta,b_\beta),\;
  (p_\gamma,j_\gamma,b_\gamma)
\]
are three triples with $\alpha<\beta<\gamma$ and $j_\alpha=j_\gamma=i$,
then $j_\beta=i$ as well.
Equivalently, the stream can be viewed as a concatenation of $m$ \emph{blocks},
one for each clustering:

\begin{gather*}
    \underbrace{(p_{1,1}, j_1, b_{1,1}), \ldots, (p_{1,t_1}, j_1, b_{1,t_1})}_{B_1} \\
    \underbrace{(p_{2,1}, j_2, b_{2,1}), \ldots, (p_{2,t_2}, j_2, b_{2,t_2})}_{B_2} \\
    \vdots \\
    \underbrace{(p_{m,1}, j_m, b_{m,1}), \ldots, (p_{m,t_m}, j_m, b_{m,t_m})}_{B_m}
\end{gather*}

where block $B_i$ contains all $O(n^2)$ triples associated with $\m{C}_{j_i}$ where $j_i \in [m]$.

We will simply call any streaming model that satisfies the contiguity property an \emph{insertion-only stream}. On the other hand, a model that does not satisfy this property -- where triples can arrive in a fully arbitrary order -- will be referred to as a \emph{generalized insertion-only stream}.

In this paper, we provide an algorithm for the $k$-median fair consensus clustering in the insertion-only streaming model. Furthermore, for a special case of $k=1$, i.e., the standard fair consensus clustering problem, our algorithm works even in the generalized insertion-only streaming model.

%% file: fair-consensus-clustering.tex
\section{Fair Consensus Clustering: A New Algorithmic Framework}\label{sec:fair_consensus_clustering} 
In this section, we provide a generic framework for solving the $ 1 $-median fair consensus clustering problem. Let $ V $ be a set of $ n $ colored points, and let $ \inpconclss = \{\clusterings\} $ be an input set of $ m $ clusterings over $ V $. The goal of a fair consensus clustering problem is to find a fair clustering $ \outconcls $ that minimizes the total distance between $ \outconcls $ and all clusterings in $ \inpconclss $, that is, 
$
    \obj{\inpconclss, \outconcls} = \sum_{i=1}^{m} \dist(\concls{i}, \outconcls). \nonumber
$

\begin{theorem}\label{thm:explicit.fair.consensus.clustering}
    Suppose that there is an $ \approxFactorClsFair $-approximation closest fair clustering with running time $ \runtimeClsFair{n} $, then there is a $ (\approxFactorClsFair+1.92) $-approximation algorithm for fair consensus clustering that runs in time $ O(m^{4}n^{2} + m^{3}\runtimeClsFair{n}) $.
\end{theorem}

\cref{thm:explicit.fair.consensus.clustering} is implied from our following key technical result, which is a framework leveraging on algorithms for closest fair clustering and fair correlation clustering.

\begin{theorem}\label{thm:fair.consensus.clustering}
Fix parameters $\approxFactorClsFair > 0$, $\approxFactorFairCorCls > 0$, $\alpha>0$, $0<\beta<1$, and $c>1$. 
Suppose there exists an algorithm that, given a clustering over $n$ points, computes a $\approxFactorClsFair$--close fair clustering in time $\runtimeClsFair{n}$, and suppose there exists a $\approxFactorFairCorCls$--approximation algorithm for fair correlation clustering on a graph with $n$ vertices that runs in time $\runtimeFairCorCls{n}$.

Then, for any collection of $m$ clusterings over $n$ points, there exists an algorithm running in time
\[
O\bigl(m^{4}n^{2} + m\,\runtimeClsFair{n} + m^{3}\runtimeFairCorCls{n}\bigr)
\]
that outputs an $(r,1)$--approximate fair consensus clustering, where
\begin{align*}
r = \max \biggl\{
   &1 + 3(\approxFactorFairCorCls+1)\alpha,\\[0.5ex]
   &\approxFactorClsFair + 2 - (\approxFactorClsFair+1)\beta,\\[0.5ex]
   &\approxFactorClsFair + 2 - \frac{2\alpha}{c},\\[0.5ex]
   &\approxFactorClsFair + 2 + (\approxFactorClsFair+1)\frac{\beta}{c-1}
        - 2\Bigl(1 - \frac{1}{c}\Bigr)\alpha
\biggr\}.
\end{align*}
\end{theorem}

We remind the reader that in a fair correlation clustering problem, we are given as input a complete graph $ G = (V, E^{+}(G)\cup E^{-}(G)) $, where any two vertices are connected by an edge that is either in $ E^{+}(G) $ (labeled ``$+$'') or in $ E^{-}(G) $ (labeled ``$-$''). The goal is to find a fair clustering $ \fairconcls $ minimizing the correlation cost, $ \costcor{\fairconcls} $, that is, the number of pairs of vertices connected by a ``$+$'' edge but placed in different clusters, and the number of pairs of vertices connected by a ``$-$'' edge but placed in the same cluster.

\textbf{Description of the algorithm.} We next present the algorithm that, given a closest fair clustering algorithm for any definition of fairness, obtains a strictly better approximation guarantee for the fair consensus clustering problem. At a high level, our algorithm (~\cref{alg.fair.consensus.clustering}) first forms a candidate set consisting of fair clusterings, and selects from this set a clustering with minimum objective value, among the candidates. We establish this candidate set by iterating through each input clustering, and add to this set a fair clustering that is close to this input. A $ \approxFactorClsFair $-close fair clustering to each input clustering $C_i \in \m{C}$ can be found by employing a $ \approxFactorClsFair $-approximation algorithm for closest fair clustering. 

Additionally, we also include in this set a fair clustering constructed as follows: for every triple of input clustering $ \{ \concls{i}, \concls{j}, \concls{k} \} $, we produce a fair clustering that maximizes the number of point pairs whose clustering relation (whether they are in the same cluster or in different clusters) agrees with the majority of the three clustering $ \{ \concls{i}, \concls{j}, \concls{k} \} $. 

Let us elaborate. A fair clustering produced from such a triple is obtained by applying the procedure $ \clsfitting $ (\cref{alg.cluster.fitting}). This procedure takes as input a set of three clusterings $ \tripleconcls = \{ \concls{i}, \concls{j}, \concls{k}  \} $, and constructs a graph $ G $ with vertex set $ V $ and edge set $ E(G) = E^{+}(G) \cup E^{-}(G) $, where
\begin{align}
    E^{+}(G) &= \left\{(a,b)\mid a,b\in V,\, 
        a \text{ and } b \text{ are together} \right. \nonumber\\
    &\quad \left. \text{in at least two} \text{ clusterings}\right\}, \text{ and} \nonumber\\
    E^{-}(G) &= \left\{(a,b)\mid a,b\in V,\, 
        a \text{ and } b \text{ are separated} \right. \nonumber\\
    &\quad \left. \text{in at least two} \text{ clusterings}\right\}.\nonumber
\end{align}
Procedure $ \clsfitting{T} $ computes a fair clustering $ \fairtriple $ by running a fair correlation clustering algorithm on $ G $.

\begin{algorithm}[htbp]
    \SetKwInOut{KwIn}{Input}
    \SetKwInOut{KwOut}{Output}
    \KwIn{Clusterings $ T = \{\concls{i}, \concls{j}, \concls{k}\} $}
    \KwOut{A fair clustering $ \fairtriple $}
        $ V(G)\gets V $

        $ E^{+}(G) \gets \{(a,b)| a,b \in V,\ a \text{ and } b \text{ are together in at least } 2 \text{ clusterings}\} $

        $ E^{-}(G) \gets \{(a,b)| a,b \in V,\ a \text{ and } b \text{ are separated in at least } 2 \text{ clusterings}\} $

        $ \fairtriple\gets $ a $ \approxFactorFairCorCls $-approximation fair correlation clustering of $ G = \left(V(G), E^{+}(G)\cup E^{-}(G)\right) $.

        \KwRet{$ \fairtriple $}
        \caption{$ \clsfitting{T} $}
    \label{alg.cluster.fitting}
\end{algorithm}

In summary,~\cref{alg.fair.consensus.clustering} computes a candidate set $ \candidateSet $, defined as
\begin{align}
    \candidateSet = &\{ \fairconcls{i}| \fairconcls{i} \text{ is a $ \approxFactorClsFair $-close fair clustering to }\concls{i}\in \inpconclss \} \nonumber \\
    \>\>\> \cup &\{ \fairtriple| \fairtriple = \clsfitting{ \tripleconcls } \text{, for each } \tripleconcls\subseteq \inpconclss \text{ such that } \card{\tripleconcls} = 3 \}. \label{eq.def.candidateSet}
\end{align}

The output of~\cref{alg.fair.consensus.clustering} is a clustering in this candidate set minimizing the objective value $ \obj{\inpconclss} $.

\begin{algorithm}[htbp]
    \SetKwInOut{KwIn}{Input}
    \SetKwInOut{KwOut}{Output}
    \KwIn{A list of clusterings $ \inpconclss = \clusterings $ over $ V $}
    \KwOut{A fair clustering $ \outconcls $}
    Initialize an empty set $ \candidateSet $\;
    \For{each $ \concls{i}\in \inpconclss  $ \nllabel{alg.line.start.loop.best.from.input}}{
        $ \fairconcls{i}\gets $ a $ \approxFactorClsFair $-close fair clustering to $ \concls{i} $ \;
        $ \candidateSet \gets \candidateSet\cup \{\fairconcls{i}\} $ \nllabel{alg.line.end.loop.best.from.input} \;
    }

    \For{each triple $ \tripleconcls = \{\concls{i}, \concls{j}, \concls{k}\} $ of $ \inpconclss $}{ \label{alg.line.start.cls.fitting}
        $ \fairtriple \gets \clsfitting{ \tripleconcls } $\;
        $ \candidateSet \gets \candidateSet\cup \{\fairtriple\} $ \nllabel{alg.line.end.cls.fitting}
    }
    \KwRet{$ \argmin_{\outconcls\in \candidateSet}\obj{\inpconclss,\outconcls} $}
    \caption{A generic framework for fair consensus clustering}
    \label{alg.fair.consensus.clustering}
\end{algorithm}

We provide the analysis in the $\loc$.

%% file: A-faster-fair-consensus.tex
\section{Streaming Algorithm for 1--median fair consensus clustering}\label{sec:streaming_one_median} 

In this section, we present an algorithm for the $1$-median fair consensus clustering
problem in the streaming setting.
Recall that, in $1$-median fair consensus clustering, we are given a set of input
clusterings $\m{I} = \{ \m{C}_1, \m{C}_2, \ldots, \m{C}_m \}$ defined over a common ground set $V$ (with $|V| = n$).
The goal is to output a fair clustering $\mathcal{F}$ that minimizes the objective
$
  \sum_{\m{C}_i \in \m{I}} \mathrm{dist}(\m{C}_i, \mathcal{F}),
$
where $\mathrm{dist}(\cdot,\cdot)$ denotes the distance between a candidate clustering and an
input clustering.

\begin{theorem}\label{thm:explicit.fair.consensus.clustering.randomized}
    Suppose that here is an $ \approxFactorClsFair $-approximation closest fair clustering with running time $ \runtimeClsFair{n} $, then there is a $ (\approxFactorClsFair+1.995) $-approximation algorithm for fair consensus clustering in the generalized insertion-only streaming model, that uses $ O(n\log m) $ space, and has a query time of $ O(n^{2}\log^{4}m + \runtimeClsFair{n}\log^{3}m) $.
\end{theorem}
The above theorem follows from fixing parameters in the following lemma; the setting of parameters is provided in the $\loc$.

\begin{lemma}\label{lem:fair.consensus.clustering.randomized}
    Suppose there is an algorithm with running time $ \runtimeClsFair{n} $ that, given a clustering over $n$ points, computes a $\approxFactorClsFair$--close fair clustering, and let there be a $ \approxFactorFairCorCls $--approximation fair correlation clustering algorithm with running time $ \runtimeFairCorCls{n} $ on a graph with $n$ vertices.
Let $\alpha>0$, $1>\beta>0$, $c > 1$, $ s>1 $ and $g > s$ be fixed parameters such that $ sc-s -2c > 0 $.
Then, there exists an algorithm that, for any set of $m$ clusterings over $n$ points and $\epsilon \in (0,1)$, with probability at least $ 1 - \frac{1}{m} $, outputs a $ \left(\frac{1+\errCoreset}{1-\errCoreset}r, 1\right) $-approximate fair consensus clustering, where
    { \small
    \begin{align}
        r = \max \Biggl\{ &(\approxFactorClsFair + 2 - (\approxFactorClsFair+1)\beta), \nonumber \\
                        &\left( \approxFactorClsFair + 2 + \left(\approxFactorClsFair+1\right)\dfrac{\beta(g-s)+s}{g(s-1)} - \dfrac{2 \alpha}{c}\right), \nonumber \\
                        &\left( \approxFactorClsFair + 2 + (\approxFactorClsFair + 1 )\dfrac{\beta(g(2c+s)-sc)+sc}{g(cs-s-2c)} - 2\left( 1- \dfrac{1}{s} - \dfrac{1}{c}  \right)\alpha \right), \nonumber\\
                        &\left( 1 + 3(\approxFactorFairCorCls+1 )\alpha \right) \Biggr\} \nonumber.
    \end{align} }
    Moreover, the algorithm runs in time
    \begin{align}
     O( n^{2}\log^{4} m + \runtimeClsFair{n} \log m + \runtimeFairCorCls{n}\log^{3}m ) 
\end{align}
and uses $ O(n\log m) $ space.
\end{lemma}

\paragraph{Algorithm Description.}

We work in the generalized insertion-only stream described in \cref{sec:preliminary}. Let us now describe our streaming algorithm for the ($1$-median) fair consensus clustering, $\sone$. 
The pseudocode and complexity analysis are presented in the $\loc$.

The algorithm maintains two memory structures in parallel:
\begin{enumerate}
    \item \textbf{Sampled Store $1$ ($\mathsf{M}_1$):}
          Prior to the arrival of the stream, we independently sample 
          $4g \log m$ indices $j_1, \ldots, j_{4g \log m}$ from $[m]$. 
          For these sampled indices, all their corresponding triples that appear in the stream 
          are stored in $\mathsf{M}_1$. 

    \item \textbf{Sampled Store $2$ ($\mathsf{M}_2$):}
          Similarly, before the stream starts, we sample 
          $64\varepsilon^{-2}\log m$ indices 
          $k_1, \ldots, k_t$ from $[m]$, where $t = 64\varepsilon^{-2}\log m$ 
          and $0 < \varepsilon < 1$. 
          For these sampled indices, all their triples are stored in $\mathsf{M}_2$ as they appear in the stream.
\end{enumerate}

After the stream ends.
\begin{enumerate}
    \item We use union-find, to construct the clusterings $\m{C}_{j_1}, \ldots, \m{C}_{j_{\log m}}$ from its triples stored in $\mathsf{M}_1$. Then we apply a subroutine $\fc$ on these $\lceil\log m \rceil$ clusterings to get a candidate set of fair clusterings $\widetilde{\m{F}}$. The subroutine $\fc$ assumes access to a $\gamma$-close fair clustering algorithm.
    \item We use union-find, to construct the clusterings $\m{C}_{k_1}, \ldots, \m{C}_{k_t}$ from its triples stored in $\mathsf{M}_2$. Let, $\m{W} = \{\m{C}_{k_1}, \ldots, \m{C}_{k_t}\}$.

    We use $\m{W}$ to find the best fair clustering $\m{F} \in \widetilde{\m{F}}$ with respect to $\m{W}$, more specifically we find
    \[
       \m{F} = \argmin_{\m{F}' \in \canset} \sum_{\m{C}_i \in \m{W}}  \dist(\m{C}_i, \m{F}')
    \]
    \item Return $\m{F}$.
\end{enumerate}

\paragraph{Analyzing the approximation factor.}
If there is a constant $ g $ such that at least $ \frac{m}{g} $ clusterings $ \concls{i} $ satisfy $ \card{\unalignedSet{i}}\leq (1-\beta)\avgconval $, then with high probability, our algorithm will sample at least one such clustering. Let $ \fairconcls $ be a $ \approxFactorClsFair $-close fair clustering to such sampled clustering, we can show that $ \obj{\inpconclss, \fairconcls}\leq (\approxFactorClsFair + 2 - (\approxFactorClsFair+1)\beta)\optconval $. Hence, in the remaining part of the analysis, we assume that the number of clusterings $ \concls{i} $ with $ \card{\unalignedSet{i}}\leq (1-\beta)\avgconval $ is less than $ \frac{m}{g} $. More specifically, we have the following ordering
\begin{align}
    \card{\unalignedSet{1}}\leq \card{\unalignedSet{2}}\leq \dots \card{\unalignedSet{t}} \leq (1-\beta)\avgconval < \card{\unalignedSet{t+1}}\leq \dots \leq \card{\unalignedSet{m}}, \label{eq.ordering.unalignedSet}
\end{align}
where $ t < \frac{m}{g} \leq t+1 $.

\begin{lemma}\label{lem:extended.S1.small.Sl.large}
    Suppose that~\eqref{eq.ordering.unalignedSet} holds. For any $ t < h\leq \frac{m}{s} $, we have
    \begin{align}
        \card{\unalignedSet{h}} \leq \left( 1 +\dfrac{\beta(g-s)+s}{g(s-1)} \right)\avgconval. \nonumber
    \end{align}
\end{lemma}
\begin{proof}
    Representing $ \optconval=m\avgconval $ as the sum off all $ \unalignedSet{i} $'s, we get that
    \begin{align}
        m\avgconval &= \sum_{i=1}^{m}\card{ \unalignedSet{i} } \geq \sum_{i=t+1}^{h-1}\card{\unalignedSet{i}} + \sum_{i=h}^{m}\card{\unalignedSet{i}} \nonumber \\
                   &> (h-t-1)(1-\beta)\avgconval + (m-h+1)\card{\unalignedSet{h}}. \nonumber 
    \end{align}
    This implies 
    \begin{align}
        \card{\unalignedSet{h} } &< \dfrac{m-(1-\beta)(h-t-1)}{m-h+1}\avgconval \nonumber \\
                                 &= \left(1 +  \dfrac{\beta(h - 1)}{m-h+1} + \dfrac{t(1-\beta)}{m-h+1} \right)\avgconval. \nonumber
    \end{align}
    As $ h\leq \frac{m}{s} $, we have $ \frac{\beta(h-1)}{m-h+1} = \frac{\beta}{(\frac{m}{h-1})-1} \leq \frac{\beta}{s-1} $. Moreover, using $ t < \frac{m}{g} $, we obtain
    \begin{align}
        \dfrac{t}{m-h+1} < \dfrac{m/g}{m-m/s+1} \leq \dfrac{s}{g(s-1)}.\nonumber
    \end{align}
    Since $ 0<\beta<1 $, we have $ \frac{t(1-\beta)}{m-h+1} < \frac{s(1-\beta)}{g(s-1)} $. Combining these two bounds, we get
    \begin{align}
        \card{\unalignedSet{h}} &< \left(1 + \dfrac{\beta}{s-1} + \dfrac{s(1-\beta)}{g(s-1)}\right)\avgconval \nonumber \\
                                 &= \left(1 + \dfrac{\beta (g-s) + s}{g(s-1)}\right)\avgconval. \nonumber
    \end{align}

\end{proof}


\begin{lemma}\label{lem:Sh.large.is.good}
    Suppose that~\eqref{eq.ordering.unalignedSet} holds. If there is a clustering $ \concls{h} $ with $ t< h\leq \frac{m}{s} $ satisfying $ \card{S_{h}} \geq \frac{m}{c} $, then
    \begin{align}
        \obj{\inpconclss, \fairconcls{h}} \leq \left( \approxFactorClsFair + 2 + \left(\approxFactorClsFair+1\right)\dfrac{\beta(g-s)+s}{g(s-1)} - \dfrac{2 \alpha}{c}\right)\optconval.\nonumber
    \end{align}
\end{lemma}
\begin{proof}
    As~\eqref{eq.ordering.unalignedSet} holds, applying~\cref{lem:extended.S1.small.Sl.large} with $ h\leq \frac{m}{s} $, we obtain $ \card{\unalignedSet{h} }\leq (1 + \frac{\beta(g-s)+s}{g(s-1)})\avgconval $. Moreover, we have 
    \begin{align}
        \obj{\inpconclss, \fairconcls{h} } &\leq \sum_{i=1}^{m}\left( \card{\unalignedSet{i} } + (\approxFactorClsFair+1 )\card{\unalignedSet{h} } - 2\card{\unalignedSet{i}\cap \unalignedSet{h} }\right) \nonumber \\
                                           &\leq \optconval + (\approxFactorClsFair+1 )\left(1+\dfrac{\beta(g-s)+s}{g(s-1)}\right)\optconval - \card{S_{h}} 2 \alpha\avgconval \nonumber \\
                                           &\leq \left( \approxFactorClsFair + 2 + (\approxFactorClsFair+1 )\dfrac{\beta(g-s)+s}{g(s-1)} - \dfrac{2 \alpha}{c} \right) \optconval. \nonumber
    \end{align}
\end{proof}


We define
\begin{align}
    \neigh{h} := \{\concls{j}| h< j \leq h + \dfrac{m}{c} + \dfrac{m}{s}:\ \card{\unalignedSet{j}\cap \unalignedSet{h} } \leq \alpha\avgconval \}, \nonumber \\
    \farNeigh{h,k} := \{ \concls{j}| \card{\unalignedSet{j}\cap \unalignedSet{h} } \leq \alpha\avgconval \text{ and } \card{\unalignedSet{j}\cap \unalignedSet{k} }\leq \alpha\avgconval \}.\nonumber
\end{align}

\begin{lemma}\label{lem:far.neigh.small.is.good}
    Suppose that~\eqref{eq.ordering.unalignedSet} holds. If for some $ h\leq \frac{m}{s} $ with $ \card{S_{h}}< \frac{m}{c} $, there exists a clustering $ \concls{k}\in \neigh{h} $ such that $ \card{\farNeigh{h,k}} < \frac{m}{s} $, then
    { \scriptsize
    \begin{align}
        \obj{\inpconclss, \fairconcls{k} } \leq \left( \approxFactorClsFair + 2 + (\approxFactorClsFair + 1 )\dfrac{\beta(g(2c+s)-sc)+sc)}{g(cs-s-2c)} - 2\left( 1- \dfrac{1}{s} - \dfrac{1}{c}  \right)\alpha \right) \optconval. \nonumber
    \end{align} 
    }
\end{lemma}
\begin{proof}
    With $ \concls{k}\in \neigh{h} $, by definition of $ \neigh{h} $, it follows that $ k\leq h + \frac{m}{c} + \frac{m}{s}\leq \frac{m}{sc/(2c+s)} $, where the last inequality is obtained by using the assumption $ h\leq \frac{m}{s} $. Applying~\cref{lem:extended.S1.small.Sl.large} with $ k \leq \frac{m}{sc/(2c+s)} $, we get $ \card{\unalignedSet{k}}\leq \left( 1+ \frac{\beta(g(2c+s)-sc)+sc}{g(sc-s-2c)} \right) \avgconval $.

    We now turn to give a lower bound for $ \card{S_{k}} $. Note that by definition of $ \farNeigh{h,k} $, for every clustering $ \concls{j}\in \inpconclss $, if $ \concls{j}\notin \farNeigh{h,k} $, then at least one of $ \card{\unalignedSet{j}\cap \unalignedSet{h} } > \alpha\optconval $ or $ \card{\unalignedSet{j}\cap \unalignedSet{k} }> \alpha\optconval $ must hold. In other words, it must be the case that $ \concls{j}\in S_{h} $ or $ \concls{j}\in S_{k} $ (or both) holds. Hence, $ \card{S_{h}} + \card{S_{k}} \geq \card{\inpconclss } - \card{\farNeigh{h,k}} $. Using the assumptions $ \card{S_{h}} < \frac{m}{c} $ and $ \card{\farNeigh{h,k} } < \frac{m}{s} $, we obtain $ \card{S_{k}}\geq m - \frac{m}{c} - \frac{m}{s} $.

    Finally, using $ \card{\unalignedSet{k} } \leq \left(1 +  \frac{\beta(g(2c+s)-sc)+sc}{g(sc-s-2c)}\right)\avgconval $ and $ \card{S_{k} }\geq m - \frac{m}{c} - \frac{m}{s} $, we have
    { \small
    \begin{align}
        &\>\>\>\>\>\obj{\inpconclss, \fairconcls{k} }  \nonumber \\
        &\leq \sum_{i=1}^{m}(\card{\unalignedSet{i} } + (1+\approxFactorClsFair )\card{\unalignedSet{k} } - 2\card{\unalignedSet{i}\cap \unalignedSet{k} }) \nonumber \\ 
                                           &\leq \optconval + (\approxFactorClsFair+1 )\left(1+ \dfrac{\beta(g(2c+s)-sc)+sc}{g(cs-s-2c)}\right)\optconval - \card{S_{k}}2 \alpha \avgconval \nonumber \\ 
                                           &\leq \left(\approxFactorClsFair + 2 + (\approxFactorClsFair+1 ) \dfrac{\beta(g(2c+s)-sc)+sc}{g(cs-s-2c)} - 2 \left(1 - \dfrac{1}{s} - \dfrac{1}{c} \right) \alpha \right) \optconval. \nonumber
    \end{align} 
    }
\end{proof}

\begin{lemma}\label{lem:cls.fitting.helps}
    Suppose that $\sone$ sampled two clusterings $ \concls{h} $ and $ \concls{k} $ such that $ k \in \neigh{h} $ and $ \card{\farNeigh{h,k}} \geq \frac{m}{s} $. Then, with probability at least $ 1 - m^{-4} $,$\sone$ also sampled a clustering $ \concls{\ell} $ such that, for $ \fairtriple = \clsfitting{\{\concls{h}, \concls{k}, \concls{\ell}\}} $, the following holds:
    \begin{align}
        \obj{\inpconclss, \fairtriple} \leq \left( 1 + 3(\approxFactorFairCorCls + 1)\alpha \right)\optconval. \nonumber
    \end{align}
\end{lemma}
\begin{proof}[Proof Sketch]
    Since $ \card{\farNeigh{h,k}}\geq \frac{m}{s} $, the probability that at least one clustering $ \concls{\ell} $ in $ \farNeigh{h,k} $ is sampled is at least
    \begin{align}
        1 - \left(1 - \dfrac{4g\log m}{m} \right)^{\frac{m}{s}} \geq 1 - e^{-4g\log m/s} = 1 - \dfrac{1}{m^{4g/s}} \geq 1 - m^{-4}. \nonumber
    \end{align}
    By definition of $ \neigh{h} $ and $ \farNeigh{h,k} $, for all $ r,s\in \{h,k,\ell\} $, we have $ \card{\unalignedSet{r}\cap \unalignedSet{s} }\leq \alpha\avgconval $. Using this, we can show that $ \obj{\inpconclss, \fairtriple }\leq (1+3(\approxFactorFairCorCls+1 ) \alpha)\optconval $. A detailed proof is provided in the $\loc$.
\end{proof}

\begin{proof}[Proof of~\cref{lem:fair.consensus.clustering.randomized}]
    If the number of clusterings $ \concls{i} $ with $ \card{\unalignedSet{i}}\leq (1-\beta)\avgconval $ is at least $ \frac{m}{g} $, then with probability at least $ 1 - m^{-4} $, our algorithm samples at least one such clustering. Let $ \fairconcls $ be a $ \approxFactorClsFair $-close fair clustering to such sampled clustering, we have $ \obj{\inpconclss, \fairconcls}\leq (\approxFactorClsFair + 2 - (\approxFactorClsFair+1)\beta)\optconval $.

    It remains to consider the case when the number of such clusterings $ \concls{i} $ is less than $ \frac{m}{g} $. With probability at least $ 1 - m^{-4(1-s/g)} $, our algorithm samples at least one clustering $ \concls{h} $ with $ \frac{m}{g} < h\leq \frac{m}{s} $. If $ \card{S_{h}}\geq \frac{m}{c} $, then by~\cref{lem:Sh.large.is.good}, let $ \fairconcls{h} $ be a $ \approxFactorClsFair $-close fair clustering to $ \concls{h} $, we have
    \begin{align}
        \obj{\inpconclss, \fairconcls{h}} \leq \left( \approxFactorClsFair + 2 + \left(\approxFactorClsFair+1\right)\dfrac{\beta(g-s)+s}{g(s-1)} - \dfrac{2 \alpha}{c}\right)\optconval. \nonumber
    \end{align}
    If $ \card{S_{h}} < \frac{m}{c} $, then $ \card{\neigh{h}}\geq \frac{m}{s} $. It follows that with probability at least $ 1 - m^{-4g/s} $, we sample a clustering $ \concls{k}\in \neigh{h} $. If $ \card{\farNeigh{h,k}} < \frac{m}{s} $, then by~\cref{lem:far.neigh.small.is.good}, let $ \fairconcls{k} $ be a $ \approxFactorClsFair $-close fair clustering to $ \concls{k} $, we have
    { \scriptsize
    \begin{align}
        \obj{\inpconclss, \fairconcls{k} } \leq \left( \approxFactorClsFair + 2 + (\approxFactorClsFair + 1 )\dfrac{\beta(g(2c+s)-sc)+sc)}{g(cs-s-2c)} - 2\left( 1- \dfrac{1}{s} - \dfrac{1}{c}  \right)\alpha \right) \optconval. \nonumber
    \end{align} 
    }
    If $ \card{\farNeigh{h,k}} \geq \frac{m}{s} $, then by~\cref{lem:cls.fitting.helps}, with probability at least $ 1 - m^{-4} $, there exists a sampled clustering $ \concls{\ell} $ such that, for $ \fairtriple = \clsfitting{\{\concls{h}, \concls{k}, \concls{\ell}\}} $, we have
    \begin{align}
        \obj{\inpconclss, \fairtriple} \leq \left( 1 + 3(\approxFactorFairCorCls + 1)\alpha \right)\optconval. \nonumber
    \end{align}

    Combining all the cases, with probability at least $ 1-m^{-3} $, there is a fair clustering $ \fairconcls\in \widetilde{\m{F}} $ such that $ \obj{\inpconclss, \fairconcls}\leq r\optconval $, where 
    { \small
    \begin{align}
        r = \max \Biggl\{ &(\approxFactorClsFair + 2 - (\approxFactorClsFair+1)\beta), \nonumber \\
                        &\left( \approxFactorClsFair + 2 + \left(\approxFactorClsFair+1\right)\dfrac{\beta(g-s)+s}{g(s-1)} - \dfrac{2 \alpha}{c}\right), \nonumber \\
                        &\left( \approxFactorClsFair + 2 + (\approxFactorClsFair + 1 )\dfrac{\beta(g(2c+s)-sc)+sc)}{g(cs-s-2c)} - 2\left( 1- \dfrac{1}{s} - \dfrac{1}{c}  \right)\alpha \right), \nonumber\\
                        &\left( 1 + 3(\approxFactorFairCorCls+1 )\alpha \right) \Biggr\} .\nonumber
    \end{align} 
    }

    As our algorithm uses a set $ \m{W} $ of $ \Omega(\varepsilon^{-2}\log m ) $ input clusterings sampled uniformly at random as a evaluation set, according to a result developed by Indyk~\cite{indyk1999sublinear},~\cite[Theorem 31]{indyk2001high}, with probability at least $ 1- \frac{1}{m} $, our algorithm outputs a clustering $ \fairconcls' $ such that $ \obj{\inpconclss, \fairconcls'} \leq (1+\varepsilon)\obj{\inpconclss, \fairconcls} $. Therefore, with probability at least $ 1 - \frac{1}{m} $, our algorithm outputs a fair clustering $ \fairconcls' $ with $ \obj{\inpconclss, \fairconcls'} \leq (1+\varepsilon)r\optconval $.
\end{proof}

%% file: k-median-consensus.tex
\section{$k$--Median Fair Consensus Clustering}\label{sec:k-median}

In this section, we formally introduce and study the \emph{$k$--median Consensus Clustering} problem.  
Given a collection of input clusterings \(\cset = \{\mathcal{C}_1, \ldots, \mathcal{C}_n\}\), the objective is to construct a set of \(k\) representative clusterings  $\oset = \{\mathcal{Z}_1, \ldots, \mathcal{Z}_k\}$ that minimizes the total distance between each input clustering and its nearest representative. Formally, we seek to minimize
\[
    \objec(\cset, \oset)
    = \sum_{\mathcal{C}_i \in \cset} \min_{\mathcal{Z}_j \in \oset} \dist(\mathcal{C}_i, \mathcal{Z}_j).
\]

In \cref{sec:fair_consensus_clustering}, given a $\gamma$-close fair clustering algorithm, we generated a $(\gamma + 1.92)$ approximation to the fair consensus clustering problem. The high-level idea of the algorithm was to generate a set of candidate fair clusterings $\canset$, which contained the union of all closest fair clusterings to each input clustering $C_i \in \m{I}$ and the set of clusterings we obtain through our cluster-fitting algorithm (\cref{alg.cluster.fitting}). We proved that one of these candidate clusterings would give us $(\gamma + 1.92)$ approximation. So, we find the best clustering $\m{C} \in \canset$ that is $\argmin_{\outconcls\in \candidateSet}\obj{\inpconclss,\outconcls}$.

Similar to this algorithm, we provide an algorithm that achieves a $(\gamma + 1.92)$ approximation to the $k$-median fair consensus clustering problem by considering $k$-sized tuples from the candidate list. Consequently, we have the following theorem, the proof of which is deferred to the $\loc$. We improve the running time in the next section when we describe the streaming variant.

\begin{theorem}\label{thm:explicit.k.fair.consensus.clustering}
    Suppose we have a  $t_1(n)$ time $\gamma$-close fair clustering algorithm and $t_2(n)$ time algorithm to find a $\eta$-approximate fair correlation clustering, then there exists an algorithm that, given a set of clusterings $\m{I} = \{ \m{C}_1, \ldots, \m{C}_m\}$ finds a set of representatives $\mathsf{Z} = \{ \m{Z}_1, \ldots, \m{Z}_k\}$ in 
    \[
        O(m \runtimeClsFair{n} + m^3 (n^2 + \runtimeFairCorCls{n}) + m^{3k + 1}n^2)
    \]
    time such that
    \[
        \objec(\m{I}, \oset) \leq (\gamma + 1.92) \objec(\m{I}, \m{Z}^*)
    \]
    where $\mathsf{Z}^* = \{ \m{Z}_1^*, \ldots, \m{Z}_k^*\}$ is the optimal set of representatives.
\end{theorem}

%% file: streaming.tex
\section{Streaming $ k $--Median Fair Consensus Clustering}
In this section, we present an algorithm for the $ k $--median fair consensus clustering problem in the insertion-only streaming model. 

\begin{theorem}\label{thm:explicit.streaming.main.theorem}
    Suppose that there is an $ \approxFactorClsFair $-approximation closest fair clustering with running time $ \runtimeClsFair{n} $, then there is a $ (1.0151\approxFactorClsFair+1.99951) $-approximation algorithm for $ k $--median fair consensus clustering in the insertion-only streaming model, which uses $ O(k^{2}n\polylog(mn)) $ space, has an update time of $ O((km)^{O(1)}n^{2}\log n) $, and a query time of $ O((k\log(mn))^{O(k)}n^{2} + k^{3}(n+\runtimeClsFair{n})\log^{12} m\log^{3} n) $.
\end{theorem}

We achieve this result by developing the following technical theorem.

\begin{theorem}\label{thm:streaming.main.theorem}
    Suppose there is an algorithm with running time $ \runtimeClsFair{n} $ that, given a clustering over $n$ points, computes a $\approxFactorClsFair$--close fair clustering, and let there be a $ \approxFactorFairCorCls $--approximation fair correlation clustering algorithm with running time $ \runtimeFairCorCls{n} $ on a graph with $n$ vertices.
Let $\alpha>0$, $ 1>\beta>0$, and $ \errCoreset,\infsmall > 0 $ be fixed parameters such that $ \alpha/2 > 2\beta(1+\beta)/(1-\beta) $. Then, for any set of $m$ clusterings over $n$ points, there exists a streaming algorithm that outputs $ ((1+\errCoreset)(r+\infsmall),k) $-approximate fair consensus clustering, where
    \begin{align}
        r = \max\bigl\{ &2+\approxFactorClsFair - ( \beta - \sampDF )(1+\approxFactorClsFair), \nonumber \\ 
                  &2+ \sampDF + \approxFactorClsFair(1+\beta+\sampDF) - \dfrac{\alpha\beta}{2(\beta+1)} + \dfrac{2\beta^{2}}{1-\beta} + o_{m}(1), \nonumber \\
                  &(1 + 3(\approxFactorClsFair+1)(\approxFactorCorCls+1)\alpha)\bigr\}.  \nonumber
    \end{align}
    The algorithm uses a space complexity of $ O(k^{2}n\polylog(mn)) $, has update time $ O((km)^{O(1)}n^{2}\log n) $, and query time $ O((k\log(mn))^{O(k)}n^{2} + k^{3}\runtimeFairCorCls{n}\log^{2} m\log^{3} n) $.
\end{theorem}

Prior to presenting the algorithm, we introduce a few tools that will be used in the algorithm and its analysis.
\begin{definition}[$ ( k,\errCoreset ) $-coreset]
    Consider $ \inpconclss $ a set of points over a metric space $ \metSpace $ equipped with a distance function $ \dist(.,.) $, and an implicit set $ \impSet\subseteq \metSpace $ (of potential median), a weighted subset $ \coreset \subseteq \inpconclss $ (with a weight function $ w:\coreset \to \R $) is a $ ( k, \errCoreset ) $-coreset of $ \inpconclss $ with respect to $ \impSet $ for the $ k $-median if for any set $ Y\subseteq \impSet $ of size $ k $
    { \footnotesize
    \begin{align}
        (1-\errCoreset )\sum_{\concls{i}\in \inpconclss} \dist(\concls{i}, Y ) \leq \sum_{\cls{P}\in \coreset} w(\cls{P})\dist( Y, \cls{P}) \leq (1+\errCoreset )\sum_{\concls{i}\in \inpconclss} \dist(\concls{i}, Y)
    \end{align} 
    }
\end{definition}
We use the following coreset constructions.
\begin{theorem}[\cite{feldman2011unified, bachem2018one, braverman2021coresets}]\label{thm:coreset.construction}
    There is an algorithm that, given a set $ \inpconclss $ of $ m $ points of an arbitrary metric space $ \metSpace $ and an implicit set $ \impSet \subseteq \metSpace $ (without loss of generality assume $ \inpconclss\subseteq \impSet $), outputs a $ \errCoreset $-coreset of $ \inpconclss $ with respect to $ \impSet $ for the $ k $--median problem, of size $ O(\errCoreset^{-2}\log \card{\impSet} ) $. The algorithm succeeds with probability at least $ 1-\frac{1}{m^{2}} $ and runs in time $ O(\errCoreset^{-2}m\log \card{\impSet}) $.
\end{theorem}
By combining~\cref{thm:coreset.construction} with the framework provided in~\cite{braverman2019streaming}, we obtain the following streaming coreset construction.
\begin{lemma}\label{lem:streaming.coreset}
    There is a (randomized) streaming algorithm that, given a set $ \inpconclss $ of $ m $ points of an arbitrary metric space $ \metSpace $, arriving in an insertion-only streaming model, and an implicit set $ \impSet \subseteq \metSpace $ (without loss of generality assume $ \inpconclss\subseteq \impSet $), maintains a $ \errCoreset $-coreset of $ \inpconclss $ with respect to $ \impSet $ for the $ k $--median problem, by storing at most $ O(\errCoreset^{-2}k\log \card{\metSpace} \log m) $ points of $ \inpconclss $. The algorithm has worst-case update time of $ ( \errCoreset^{-1}k\log m )^{O(1)} $.
\end{lemma}

Another sampling technique we use is the \emph{monotone faraway sampling} introduced by~\cite{braverman2021metric}.
\begin{lemma}[Monotone Faraway Sampling~\cite{braverman2021metric}]\label{lem:monotone.faraway.sampling}
    There is a (randomized) streaming algorithm that, given a set $ \inpconclss $ of $ m $ points of an arbitrary metric space $ \metSpace $, arriving in an insertion-only streaming model, and parameters $ \mfskappa, \mfsrho \in (0,1) $, samples a subset $ \mfsSet \subseteq \inpconclss $ of size $ O(k^{2}(\mfsrho\mfskappa)^{-1}\log k\log(1+k\mfskappa m)) $ such that the following holds: Let $ \fairmediank = \{ \fairmediank{1}, \dots, \fairmediank{k}\} $ be an arbitrary optimum $ k $--median of $ \inpconclss $, and let $ \optclsk{1}, \dots, \optclsk{k} $ be the clustering of $ \inpconclss $ induced from $ \fairmediank $. Then for each $ i\in [k] $, there exists a $ \concls{i}'\in \mfsSet $ satisfying
    \begin{align}
        \sum_{\concls\in \optclsk{i}}\dist(\concls', \concls{i}') \leq 2\left(1 + \dfrac{1}{1-\mfskappa}\right) \sum_{\concls\in \optclsk{i}}\dist(\concls, \fairmediank{i}) + \mfsrho \dfrac{\optconval}{k}.\nonumber
    \end{align}
    The algorithm requires both space and update time of $ O(k^{2}(\mfsrho\mfskappa)^{-1}\log k\log(1+k\mfskappa m)) $.
\end{lemma}
Algorithms in~\cref{lem:streaming.coreset} and~\cref{lem:monotone.faraway.sampling} succeeds with probability at least $ \frac{9}{10} $. Note that in our consensus clustering problem, the underlying metric space $ \metSpace $ consists of the set of all clusterings over $V$, and each input clustering (in $ \inpconclss $) is a point in the metric space $ \metSpace $.

\paragraph{Description of the Algorithm.} Our algorithm consists of four components. The first three components process simultaneously the input clusterings as they arrive in the stream to establish a candidate set and construct a coreset. By using such sets, the last component simulates the offline algorithm presented in~\cref{sec:k-median} to output $ k $ fair clusterings as $ k $-median.

$ \bullet $\textbf{Step 1.} We simultaneously do the following steps as the clusterings from the input set $ \inpconclss $ arrives in the stream.

\>\>\setword{\textbf{Step 1A}}{StepOneA} \textbf{(Sampling Algorithm)} We choose constants $ \sampDF, \rate>0 $, the value of which to be fixed later. For each $ \sampDist \in \{\frac{1}{2},\frac{1}{2}(1+\rate),\frac{1}{2}(1+\rate)^{2}, \dots, n^{2}\} $ and $ p \in \{1, \frac{1}{1+\rate}, \frac{1}{(1+\rate){2}}, \dots, \frac{1}{m} \} $, we construct a set $ \sampSet{\sampDist, \sampRate}\subset \inpconclss $ as follows.
\begin{enumerate}[label = \textbf{Step \alph* }, leftmargin=1.8 cm]
    \item For each arrived $ \concls{i} $, discard it with probability $ 1-p $.\label{enu.sample}
    \item If $ \dist(\concls{i}, \concls{j})\geq \sampDF\sampDist $, for all $ \concls{j}\in \sampSet{\sampDist, \sampRate} $, add $ \concls{i} $ to $ \sampSet{\sampDist, \sampRate} $. \label{enu.add.to.sample.set}
    \item If $ \card{\sampSet{\sampDist, \sampRate}}\geq k\log^{3} m $, set $ \sampSet{\sampDist, \sampRate} = \emptyset $. \label{enu.remove.large.sample.set}
\end{enumerate}
Set $ \sampSet = \cup_{\sampDist, \sampRate}\sampSet{\sampDist, \sampRate} $.

\>\>\setword{\textbf{Step 1B}}{StepOneB} \textbf{(Monotone Faraway Sampling)} As the clusterings from the input set $ \inpconclss $ arrives in the stream, we run the algorithm from~\cref{lem:monotone.faraway.sampling} with parameters $ \mfskappa = 1/3 $ and a constant $ \mfsrho>0 $. Let $ \mfsSet $ be the output of this algorithm. Then, for any arbitrary optimal $ k $--median clusterings $ \fairmediank = \{ \fairmediank{1}, \dots, \fairmediank{k}\} $ of $ \inpconclss $, and the corresponding super-clusters $ \optclsk{1}, \dots, \optclsk{k} $, for each $ i\in [k] $, there exists a $ \concls{i}'\in \mfsSet $ satisfying
\begin{align}
    \sum_{\concls\in \optclsk{i}}\dist(\concls, \concls{i}') \leq 5 \sum_{\concls\in \optclsk{i}}\dist(\concls, \fairmediank{i}) + \mfsrho \dfrac{\optconval}{k}.\nonumber
\end{align}

\>\>\setword{\textbf{Step 1C}}{StepOneC} \textbf{(Coreset Construction)} We consider the candidate set $ \candidateSet $ as in~\eqref{eq.def.candidateSet}. For each $ \concls{i}\in \inpconclss $, $ \candidateSet $ contains $ \fairconcls{i} $, a $ \approxFactorClsFair $-close fair clustering to $ \concls{i} $. For every triple $ \tripleconcls=(\concls{i}, \concls{j}, \concls{k}) $, $ \candidateSet $ contains $ \fairtriple = \clsfitting{\{\tripleconcls\}} $. 

\>\>As the clusterings from the input set $ \inpconclss $ arrives in the stream, we run the algorithm from~\cref{lem:streaming.coreset}
with a constant $ \errCoreset > 0 $ to build a $ (k,\errCoreset) $-coreset $ (\coreset,w) $ with respect to the implicit set $ \candidateSet $.

$ \bullet $ \textbf{Step 2 (Simulating the Offline Algorithm)}
We use a candidate set $ \rcandidateSet $ by iterating through each clustering $ \concls{i} \in \sampSet\cup \mfsSet $, including $ \fairconcls{i} $ in $ \rcandidateSet $, where $ \fairconcls{i} $ is a $ \approxFactorClsFair $-close fair clustering to $ \concls{i} $. Moreover, for every triple $ \tripleconcls = \{ \concls{i}, \concls{j}, \concls{k} \} \subseteq \sampSet $, we include $ \fairtriple = \clsfitting{\tripleconcls} $ in $ \rcandidateSet $. Finally, for each $ k $-tuple $ Y = (\cls{Y}_{1}, \dots, \cls{Y}_{k})\in \rcandidateSet^{k} $, we compute the objective $ \obj{\coreset, Y} $ and return the $ k $-tuple minimizing this value. We defer the entire analysis to the $\loc$.

%% file: appendix.tex
\section{Missing Details from Section~\ref{sec:fair_consensus_clustering}}
\textbf{Running time analysis.} The input contains $ m $ clusterings, and computing a $ \approxFactorClsFair $-close fair clustering to each of them takes $ \runtimeClsFair{n} $ time. Hence, the execution from~\cref{alg.line.start.loop.best.from.input} to~\cref{alg.line.end.loop.best.from.input} takes $ O(m\runtimeClsFair{n}) $ time. There are at most $ O(m^{3}) $ subset of $ \inpconclss $ of size exactly $ 3 $. For each such subset, $ \clsfitting $ requires $ O(n^{2}) $ time to construct the graph $ G $. Applying a fair correlation clustering algorithm for $ G $ to obtain $ \fairtriple $ runs in $ \runtimeFairCorCls{n} $ time. Therefore, procedure $ \clsfitting $ runs in $ O(n^{2}+ \runtimeFairCorCls{n}) $ time. Summing up, computing the set $ \candidateSet $ (~\cref{alg.line.start.loop.best.from.input} to~\cref{alg.line.end.cls.fitting}) takes $ O(m\runtimeClsFair{n} + m^{3}(n^{2}+ \runtimeFairCorCls{n})) $ time. 

To determine $ \argmin_{\outconcls\in \candidateSet}\obj{\inpconclss, \outconcls} $, we calculate $ \obj{\inpconclss, \outconcls} $ for each $ \outconcls\in \candidateSet $. Each such calculation runs in $ O(mn^{2}) $ time, as it requires $ O(n^{2}) $ to compute $ \dist( \outconcls, \concls{i} ) $ for each $ \concls{i}\in \inpconclss $. Since there are $ O(m^{3}) $ candidate clusterings in $ \candidateSet $, the step determining $ \argmin_{\outconcls\in \candidateSet}\obj{\inpconclss, \outconcls} $ runs in $ O(m^{4}n^{2}) $ time.

Overall, the running time of~\cref{alg.fair.consensus.clustering} is $ O(m^{4}n^{2} + m\runtimeClsFair{n} + m^{3}\runtimeFairCorCls{n}) $ time.

\textbf{Analyzing the approximation factor.} Let $ \optconval $ denote the minimum value of $ \obj{\inpconclss} $. Let $ \alpha \geq 0,\ 1>\beta \geq 0$ and $ c > 1 $ be some positive parameters. Recall that $ \approxFactorClsFair $ is the approximation ratio of the algorithm for closest fair clustering, and $ \approxFactorFairCorCls $ is the approximation ratio of the algorithm for fair correlation clustering. It suffices to show that there exists a clustering $ \outconcls $ in the candidate set $ \candidateSet $ such that $ \obj{\inpconclss, \outconcls} \leq r\cdot\optconval $. This bound is derived as a direct consequence of the following key lemma.

\begin{lemma}\label{lem:approximation.factor.consensus.clustering}
    Let $ \candidateSet $ be defined as in~\eqref{eq.def.candidateSet}. Then one of the following holds:
    \begin{enumerate}
        \item There is a set of three clusterings $ \tripleconcls = \{\concls{i}, \concls{j}, \concls{k} \} $ such that if $ \fairtriple = \clsfitting{ \tripleconcls } $, then $ \obj{\inpconclss, \fairtriple} \leq (1+3(\approxFactorFairCorCls+1)\alpha)\optconval $;
        \item There is a clustering $ \concls{i} $ such that $ \obj{\inpconclss, \fairconcls{i}} \leq \max\{\approxFactorClsFair + 2 - (\approxFactorClsFair+1)\beta, \approxFactorClsFair + 2 -\frac{2\alpha}{c}, \approxFactorClsFair +2 +(\approxFactorClsFair+1)\frac{\beta}{c-1}-2\left(1- \frac{1}{c}\right)\alpha\}\optconval $.
    \end{enumerate}
\end{lemma}
We begin by stating some definitions and proving auxiliary lemmas, from which~\cref{lem:approximation.factor.consensus.clustering} follows.

Fix an optimal fair clustering $ \optconcls $, then by definition, $ \obj{\inpconclss, \optconcls} = \optconval $. Denote $ \avgconval = \frac{\optconval}{m} $.

For each clustering $ \concls{i} $, define $ \unalignedSet{i} $ to be the set of pairs of points $ (a,b) $ in $ V $ such that if $ a $ and $ b $ are together (resp.\ separated) in $ \optconcls $, then $ a $ and $ b $ are separated (resp.\ together) in $ \concls{i} $. In other words, by definition, $ \dist(\concls{i}, \optconcls) = \card{\unalignedSet{i}} $. 

\begin{lemma}\label{lem:cluster.fitting.helps}
    If there is a set $ T = \{\concls{i}, \concls{j}, \concls{k}\} $ such that, for all $ r\neq s \in \{i,j,k\} $, we have $ \card{\unalignedSet{r} \cap \unalignedSet{s}} \leq \alpha\avgconval $, then $ \obj{\inpconclss, \fairtriple} \leq (1 + 3(\approxFactorFairCorCls+1)\alpha)\optconval $.
\end{lemma}
\begin{proof}
    We define a set of bad pairs
    \begin{align}
        B = (\unalignedSet{i} \cap \unalignedSet{j}) \cup (\unalignedSet{j} \cap \unalignedSet{k}) \cup (\unalignedSet{k} \cap \unalignedSet{i}), \nonumber
    \end{align}
    that is, $ B $ contains all pairs of points $ (a,b) $ such that $ a,b $ are together (resp.\ separated) in a majority of clusterings in $ T $, while in $ \optconcls $, $ a $ and $ b $ are separated (resp.\ together). 

    Note that $ \card{B} \leq 3\alpha\avgconval $, as $ \card{\unalignedSet{r} \cap \unalignedSet{s}} \leq \alpha\avgconval $, for all $ r\neq s \in \{i,j,k\} $.

    Let $ G $ be a graph with vertex set $ V $ and the edge set $ E(G) = E^{+}(G)\cup E^{-}(G) $, where
    \begin{align}
    E^{+}(G) &= \left\{(a,b)\mid a,b\in V,\, 
        a \text{ and } b \text{ are together} \right. \nonumber\\
             &\quad \left. \text{in at least } 2 \text{ clusterings among } \concls{i}, \concls{j}, \concls{k}\right\}, \text{ and} \nonumber\\
    E^{-}(G) &= \left\{(a,b)\mid a,b\in V,\, 
        a \text{ and } b \text{ are separated} \right. \nonumber\\
             &\quad \left. \text{in at least } 2 \text{ clusterings among } \concls{i}, \concls{j}, \concls{k} \right\}.\nonumber
\end{align}
    Observe that for each pair of vertices $ (a,b) $, if $ (a,b)\notin B $, then $ a,b $ are together (resp.\ separated) in $ \optconcls $ if $ (a,b) $ is an edge in $ E^{+}(G) $ (resp.\ $ E^{-}(G) $). On the other hand, if $ (a,b)\in B $, then $ a,b $ are separated (resp.\ together) if $ (a,b) $ is an edge in $ E^{+}(G) $ (resp.\ $ E^{-}(G) $). Therefore, if for every pair $ (a,b)\in B $, we flip their labels (from ``$+$'' to  ``$-$'' and vice versa) we obtained a new graph $ G' $ such that, any two vertices in a same cluster in $ \optconcls $ are connected by an ``$+$'' edge, and any two vertices is different clusters in $ \optconcls $ are connected by a ``$-$'' edge. Note that, the correlation cost of a clustering $ \concls $ of $ G $ can be viewed as the number of pairs $ (a,b) $ whose labels need to be flipped so that, the edges inside every cluster of $ \concls $ are all ``$+$'' edges, and the edges between the clusters of $ \concls $ are all ``$-$'' edges. Hence, $ \optconcls $ has a  correlation cost $ \costcor{\optconcls} = \card{B} $.

    Let $ \fairtriple $ be the output of $ \clsfitting{\tripleconcls} $ (\cref{alg.cluster.fitting}). In other words, $ \fairtriple $ is a clustering obtained by running a $ \approxFactorFairCorCls $-approximation fair correlation clustering algorithm on $ G $. Then $ \costcor{\fairtriple}\leq \approxFactorFairCorCls\card{B} $, as the cost of an optimal fair correlation clustering is at most $ \card{B} $. Moreover, the number of pairs $ (a,b) $ that are clustered differently between $ \fairtriple $ and $ \optconcls $ is at most, every pair in $ B $ and the $ \approxFactorFairCorCls\card{B} $ pairs that are flipped by the approximation fair correlation clustering algorithm. As a result, we get

    \begin{align}
        \dist(\fairtriple, \optconcls) &\leq \card{B} + \approxFactorFairCorCls\card{B} \nonumber \\
                                          &\leq (1+\approxFactorFairCorCls)3 \alpha\avgconval &( \text{as }\card{B}\leq 3 \alpha\avgconval ). \nonumber
    \end{align}
    It remains to bound the objective value of $ \fairtriple $.
    \begin{align}
        \obj{\inpconclss, \fairtriple} &= \sum_{i=1}^{m}\dist(\concls{i}, \fairtriple) \nonumber \\
                                       &\leq \sum_{i=1}^{m}(\dist(\concls{i}, \optconcls) + \dist(\optconcls, \fairtriple) ) \nonumber \\
                                       &\leq \optconval + m(1+\approxFactorFairCorCls)3 \alpha\avgconval \nonumber \\
                                       &= (1+3(\approxFactorFairCorCls+1) \alpha)\optconval. \nonumber
    \end{align}
\end{proof}

For the rest of this section, we can assume that, for all triples $ \{\concls{i}, \concls{j}, \concls{k}\} \subseteq \inpconclss $:
\begin{align}
     \exists r\neq s\in \{i,j,k\} \text{ such that } \card{\unalignedSet{r}\cap \unalignedSet{s}}> \alpha\avgconval. \label{eq.no.bad.set.condition}
\end{align}
\begin{lemma}\label{lem:if.unaligned.is.small}
    If there is a clustering $ \concls{i} $ such that $ \card{\unalignedSet{i}}\leq (1-\beta)\avgconval $, then $ \obj{\inpconclss, \fairconcls{i}}\leq (\approxFactorClsFair + 2 - (\approxFactorClsFair+1)\beta)\optconval $.
\end{lemma}
\begin{proof}
    By the triangle inequality, we have
    \begin{align}
        &\>\>\>\>\>\obj{\inpconclss, \fairconcls{i}} = \sum_{j=1}^{m}\dist(\concls{j}, \fairconcls{i}) \nonumber \\
                                          &\leq \sum_{j=1}^{m}(\dist(\concls{j},\optconcls) + \dist(\optconcls, \concls{i})+ \dist(\concls{i}, \fairconcls{i}))  \nonumber \\
                                          &\leq \optconval + m\dist(\optconcls, \concls{i}) + m\approxFactorClsFair \dist(\optconcls, \concls{i}) \nonumber \\
                                          &\leq \optconval + (1+\approxFactorClsFair)(1-\beta)m\avgconval  \nonumber \\
                                          &= (\approxFactorClsFair + 2 - (\approxFactorClsFair+1)\beta)\optconval, \nonumber
    \end{align}
    where the second inequality is due to the fact that $ \fairconcls{i} $ is a $ \approxFactorClsFair $-close fair clustering to $ \concls{i} $, and the last inequality follows from the assumption that $ \dist(\optconcls,\concls{i}) = \card{\unalignedSet{i}}\leq (1-\beta)\avgconval $.
\end{proof}
From now on, we assume that
\begin{align}
    \card{\unalignedSet{i}} > (1-\beta)\avgconval, \text{ for all } \unalignedSet{i}.\label{eq.no.small.unalignedSet}
\end{align}
Moreover, without loss of generality, we can assume $ \card{\unalignedSet{1}}\leq \card{\unalignedSet{2}} \leq \dots \leq \card{\unalignedSet{m}} $. It follows that $ \card{\unalignedSet{1}}\leq \avgconval $. 

Let $ \ell $ be the smallest index such that $ \card{\unalignedSet{1}\cap \unalignedSet{\ell}}\leq \alpha\avgconval $ (if such an $ \ell $ does not exists, set $ \ell = m+1 $). Then for all $ i < \ell $, $ \card{\unalignedSet{i}\cap \unalignedSet{1}}> \alpha\avgconval $. Define
\begin{align}
    S_{i} = \{\concls{j}| \card{\unalignedSet{i}\cap \unalignedSet{j}} > \alpha\avgconval\}.\nonumber
\end{align}
The next lemma establishes that if $ \card{S_{1}} $ is large, then $ \fairconcls{1} $ gives us a good approximation solution. Prior to that, we provide the following observation, which will be used in the proofs.
\begin{align}
    \dist(\concls{i}, \concls{j}) = \card{\unalignedSet{i}} + \card{\unalignedSet{j}} - 2 \card{\unalignedSet{i} \cap \unalignedSet{j}}. \label{eq.dist.unalignedSet.relation}
\end{align}
\begin{lemma}\label{lem:first.closest.fair.clustering.is.good}
    Suppose that~\eqref{eq.no.bad.set.condition} and~\eqref{eq.no.small.unalignedSet} hold. Let $ c > 1 $ satisfy $ \card{S_{1}}\geq \frac{m}{c} $, then
    \begin{align}
        \obj{\inpconclss, \fairconcls{1}} \leq (\approxFactorClsFair + 2 - \dfrac{2 \alpha}{c}) \optconval. \nonumber
    \end{align}
\end{lemma}
\begin{proof}
    Recall that $ \fairconcls{1} $ is a $ \approxFactorClsFair $-close fair clustering to $ \concls{1} $, hence, $ \dist(\concls{1}, \fairconcls{1}) \leq \approxFactorClsFair \dist(\concls{1}, \optconcls) = \approxFactorClsFair \card{\unalignedSet{1}} $. By triangle inequality and by~\cref{eq.dist.unalignedSet.relation}, we have
    \begin{align}
        \obj{\inpconclss, \fairconcls{1}} &= \sum_{i=1}^{m}\dist(\concls{i}, \fairconcls{1}) \nonumber \\
                                          &\leq \sum_{i=1}^{m} (\dist(\concls{i}, \concls{1})+\dist(\concls{1}, \fairconcls{1})) \nonumber \\
                                          &\leq \sum_{i=1}^{m}(\card{\unalignedSet{i}}+\card{\unalignedSet{1}} - 2\card{\unalignedSet{i}\cap \unalignedSet{1}} + \approxFactorClsFair \dist(\concls{1}, \optconcls))  \nonumber \\
                                          &= \sum_{i=1}^{m}(\card{\unalignedSet{i}}+(\approxFactorClsFair+1)\card{\unalignedSet{1}}) - 2 \sum_{\concls{i}\in S_{1}}\card{\unalignedSet{i}\cap \unalignedSet{1}} \nonumber \\
                                          &\leq (\approxFactorClsFair+2)\optconval - \dfrac{2\alpha}{c}\optconval \nonumber\\
                                          &= \left(\approxFactorClsFair + 2 - \dfrac{2 \alpha}{c} \right)\optconval.\nonumber
    \end{align}
    The last inequality is due to $ \card{\unalignedSet{1}}\leq \avgconval $,  $ \card{\unalignedSet{1}\cap \unalignedSet{i}} >\alpha \avgconval $ for all $ \concls{i}\in S_{1} $, and $ \card{S_{1}}\geq \frac{m}{c} $.
\end{proof}

In the remaining case, when $ \card{S_{1}}< \frac{m}{c} $, we demonstrate that $ \fairconcls{\ell} $ is a good approximation clustering. To illustrate that, we need an upper bound for $ \card{I_{\ell}} $, which is obtained by the following lemma.

\begin{lemma}\label{lem:S1.small.Sl.large}
    Suppose that~\eqref{eq.no.small.unalignedSet} holds. Then for any constant $ c > 1 $ and an index $ h $ satisfying $ h\leq \frac{m}{c} $, we have $ \card{I_{h}}\leq (1+\frac{\beta}{c-1})\avgconval $.
\end{lemma}
\begin{proof}
    Representing $ \optconval=m\avgconval $ as the sum off all $ \unalignedSet{i} $'s, we get that
    \begin{align}
        m\avgconval &= \sum_{i=1}^{m}\card{ \unalignedSet{i} } = \sum_{i=1}^{h-1}\card{\unalignedSet{i}} + \sum_{i=h}^{m}\card{\unalignedSet{i}}\\
                   &> (h-1)(1-\beta)\avgconval + (m-h+1)\card{\unalignedSet{h}}. \nonumber 
    \end{align}
    This implies 
    \begin{align}
        \card{\unalignedSet{h} } &< \dfrac{m-(1-\beta)(h-1)}{m-h+1}\avgconval = \left(1 +  \dfrac{\beta(h - 1)}{m-h+1} \right)\avgconval \nonumber \\
                                    &= \left(1+\dfrac{\beta}{\frac{m}{h-1}-1}\right) \avgconval. \nonumber
    \end{align}
    This completes the proof, since the assumption $h \leq \frac{m}{c}$ implies that $\frac{m}{h-1} > c$.
\end{proof}

From the definitions of $ S_{1} $ and $ \ell $, we have $ \ell-1\leq \card{S_{1}}\leq \frac{m}{c} $. Applying~\cref{lem:S1.small.Sl.large} with $ h = \ell $, we obtain the $ \card{\unalignedSet{\ell}} \leq (1+ \frac{\beta}{c-1}\avgconval) $. We are ready to demonstrate that $ \fairconcls{\ell} $ is a good approximation clustering.

\begin{lemma}\label{lem:ell.closest.fair.clustering.is.good}
    Suppose that~\eqref{eq.no.bad.set.condition} and~\eqref{eq.no.small.unalignedSet} hold. Let $ c > 1 $ satisfy $ \card{S_{1}} < \frac{m}{c} $. Then 
    \begin{align}
        \obj{\inpconclss, \fairconcls{\ell}} \leq \left(3+\dfrac{2\beta}{c-1}-2\left(1- \dfrac{1}{c}\right)\alpha\right)\optconval.\nonumber
    \end{align}
\end{lemma}
\begin{proof}
    By definition of $ \ell $, $ \card{\unalignedSet{1}\cap \unalignedSet{\ell}}\leq \alpha\avgconval $. From~\eqref{eq.no.bad.set.condition}, it follows that for every $ \concls{i}\in \inpconclss $, it must be that $ \card{\unalignedSet{i}\cap \unalignedSet{1}}> \alpha\avgconval\ (\concls{i}\in S_{1}) $ or $ \card{\unalignedSet{i}\cap \unalignedSet{\ell}}> \alpha\avgconval\ (\concls{i}\in S_{\ell}) $. Therefore, $ \card{S_{1}}+ \card{S_{\ell}} \geq m $, which leads to $ \card{S_{\ell}}\geq m(1-\frac{1}{c}) $, as $ \card{S_{1}}<\frac{m}{c} $. By the triangle inequality, we have
    \begin{align}
        &\>\>\>\>\>\>\obj{\inpconclss, \fairconcls{\ell}} = \sum_{i=1}^{m} \dist(\concls{i}, \fairconcls{\ell}) \nonumber \\ 
        &\leq \sum_{i=1}^{m}\left(\dist(\concls{i},\concls{\ell}) + \dist(\concls{\ell}, \fairconcls{\ell})\right) \nonumber \\
        &\leq \sum_{i=1}^{m} (\card{\unalignedSet{i}}+(\approxFactorClsFair + 1)\card{\unalignedSet{\ell}}-2\card{\unalignedSet{i}\cap \unalignedSet{\ell}}) &  \nonumber \\ 
        &\leq \optconval + (\approxFactorClsFair + 1)\left(1+\dfrac{\beta}{c-1}\right)\optconval - \card{S_{\ell}}2 \alpha\avgconval \nonumber \\ 
        &\leq \left(\approxFactorClsFair + 2 +(\approxFactorClsFair+1 )\dfrac{\beta}{c-1}-2\left(1- \dfrac{1}{c}\right)\alpha\right)\optconval ,\nonumber
    \end{align}
    where the second inequality follows from~\cref{eq.dist.unalignedSet.relation} and $ \dist(\concls{\ell}, \fairconcls{\ell})\leq \approxFactorClsFair\card{\unalignedSet{\ell}} $, the third inequality is due to $ \card{\unalignedSet{\ell}} \leq (1+ \frac{\beta}{1+c})\avgconval $ and the definition of $ S_{\ell} $, and the last inequality follows from $ \card{S_{\ell}}\geq m(1-\frac{1}{c}) $.
\end{proof}
\begin{proof}[Proof of~\cref{lem:approximation.factor.consensus.clustering}]
    If there is a set of three clustering $ T = \{\concls{i}, \concls{j}, \concls{k}\}\subseteq \inpconclss $ such that for all $ r\neq s\in \{i,j,k\} $: $ \card{\unalignedSet{r}\cap \unalignedSet{s}}\leq \alpha\avgconval $, then by~\cref{lem:cluster.fitting.helps}, $ \obj{\inpconclss, \fairtriple} \leq (1 + 3(\approxFactorFairCorCls+1)\alpha)\optconval $. Otherwise, by~\cref{lem:if.unaligned.is.small},~\cref{lem:first.closest.fair.clustering.is.good}, and~\cref{lem:ell.closest.fair.clustering.is.good}, there is a clustering $ \concls{i}\in \inpconclss $ satisfying $ \obj{\inpconclss, \fairconcls{i}} \leq \max(\approxFactorClsFair + 2 - (\approxFactorClsFair+1)\beta, \approxFactorClsFair + 2 -\frac{2\alpha}{c}, \approxFactorClsFair +2 +(\approxFactorClsFair+1)\frac{\beta}{c-1}-2\left(1- \frac{1}{c}\right)\alpha)\optconval $, where $ \fairconcls{i} $ is a $ \approxFactorClsFair $-close fair clustering to $ \concls{i} $.
\end{proof}

Before proving~\cref{thm:explicit.fair.consensus.clustering}, we describe how to obtain a $ \approxFactorFairCorCls = (\approxFactorClsFair\approxFactorCorCls + \approxFactorClsFair + \approxFactorCorCls) $-approximation correlation clustering algorithm, provided that we have a $ \approxFactorCorCls $-approximation correlation clustering algorithm.

\begin{lemma}\label{lem:approx.fair.correlation.clustering}
    Suppose that $ G $ is a complete graph with the edge set $ E(G) = E^{+}(G)\cup E^{-}(G) $. Let $ \tripleconcls' $ be a $ \approxFactorCorCls $-approximation correlation clustering of $ G $, and let $ \fairtriple $ be a $ \approxFactorClsFair $-close fair clustering to $ \tripleconcls' $. Then $ \fairtriple $ is $ (\approxFactorClsFair\approxFactorCorCls + \approxFactorClsFair + \approxFactorCorCls) $-approximation fair correlation clustering of $ G $.
\end{lemma}
\begin{proof}
    \newcommand{\corclsopt}{\cls{C}^{*}_{\mathrm{cor}}}
    \newcommand{\corclsfairopt}{\cls{C}^{*}_{\mathrm{cor,fair}}}

    Let $ \corclsopt $ be an optimal correlation clustering of $ G $, and let $ \corclsfairopt $ be an optimal fair correlation clustering of $ G $. By definition, we have $ \costcor{\tripleconcls'}\leq \approxFactorCorCls\costcor{\corclsopt} \leq \approxFactorCorCls\costcor{\corclsfairopt} $.

    With $ \fairtriple $ being a $ \approxFactorClsFair $-close fair clustering to $ \tripleconcls' $, we have $ \dist(\fairtriple, \tripleconcls')\leq \approxFactorClsFair \dist(\corclsfairopt, \tripleconcls') $. Note that, for any two clusterings $ \concls{A} $ and $ \concls{B} $, we have $ \dist(\concls{A}, \concls{B}) \leq \costcor{\concls{A}} + \costcor{\concls{B}} $ and $ \costcor{\concls{A}} \leq \costcor{\concls{B}} + \dist(\concls{A}, \concls{B}) $. Therefore,
    \begin{align}
        \costcor{\fairtriple} &\leq \costcor{\tripleconcls'} + \dist(\fairtriple, \tripleconcls') \nonumber \\
                             &\leq \costcor{\tripleconcls'} + \approxFactorClsFair\dist(\corclsfairopt, \tripleconcls') \nonumber \\
                             &\leq \costcor{\tripleconcls'} + \approxFactorClsFair(\costcor{\corclsfairopt} + \costcor{\tripleconcls'}) \nonumber \\
                             &\leq (\approxFactorClsFair + 1)\costcor{\tripleconcls'} + \approxFactorClsFair\costcor{\corclsfairopt} \nonumber \\
                             &\leq (\approxFactorClsFair + 1)\approxFactorCorCls\costcor{\corclsfairopt} + \approxFactorClsFair\costcor{\corclsfairopt} \nonumber \\
                             &= (\approxFactorClsFair\approxFactorCorCls + \approxFactorClsFair + \approxFactorCorCls)\costcor{\corclsfairopt}. \nonumber
    \end{align}
    It follows that $ \fairtriple $ is a $ (\approxFactorClsFair\approxFactorCorCls + \approxFactorClsFair + \approxFactorCorCls) $-approximation fair correlation clustering of $ G $.
\end{proof}

A $ (1.437+\epsilon) $-approximation correlation clustering algorithm with running time $ O(n^{\poly( 1/\epsilon )}) $ was shown in~\cite[Theorem 3]{cao2024understanding}, and later improved to $ \OO(2^{\poly(1/\epsilon)}n) $ in~\cite[Corrollary 3]{cao2025solving}. Thus, we obtain the following corollary.
\begin{corollary}\label{cor:approx.fair.correlation.clustering}
    There is a $ \approxFactorFairCorCls = (\approxFactorClsFair\approxFactorCorCls +  \approxFactorClsFair + \approxFactorCorCls) $-approximation fair correlation clustering algorithm that runs in time $ \runtimeFairCorCls{n} = (\runtimeClsFair{n} + n) $, with $ \approxFactorCorCls = 1.4371 $.
\end{corollary}

We conclude this section by proving~\cref{thm:explicit.fair.consensus.clustering}.
\begin{proof}[Proof of~\cref{thm:explicit.fair.consensus.clustering}]
    We apply~\cref{thm:fair.consensus.clustering} and~\cref{cor:approx.fair.correlation.clustering} with $ c = 3 $, $ \alpha = \frac{3}{10(\approxFactorCorCls+1)} $, and $ \beta = \frac{2}{5(\approxFactorClsFair+1)(\approxFactorCorCls+1)} \in (0,1) $. Note that
    \begin{itemize}
        \item $ 1+3(\approxFactorFairCorCls+1)\alpha = 1+ 3(\approxFactorClsFair+1)(\approxFactorCorCls+1)\alpha = 9/(10(\approxFactorClsFair+1)) + 1< \approxFactorClsFair + 1.9 $.
        \item $ \approxFactorClsFair + 2 - (\approxFactorClsFair+1)\beta = \approxFactorClsFair + 2 - \frac{2}{5(\approxFactorCorCls+1)} < \approxFactorClsFair + 1.84 $, as $ \approxFactorCorCls = 1.4371 $.
        \item $ \approxFactorClsFair + 2 - \frac{2\alpha}{c} = \approxFactorClsFair + 2 - \frac{1}{5(\approxFactorCorCls+1)} < \approxFactorClsFair + 1.92 $, as $ \approxFactorCorCls = 1.4371 $.
    \item $ \approxFactorClsFair + 2 + (\approxFactorClsFair+1)\frac{\beta}{c-1} - 2(1-\frac{1}{c})\alpha = \approxFactorClsFair + 2 + \frac{(\approxFactorClsFair+1)\beta}{2} - \frac{4\alpha}{3} = \approxFactorClsFair + 2 + \frac{(\approxFactorClsFair+1)}{5(\approxFactorClsFair+1)(\approxFactorCorCls+1)} - \frac{4}{10(\approxFactorCorCls+1)} = \approxFactorClsFair + 2 - \frac{1}{5(\approxFactorCorCls+1)} < \approxFactorClsFair + 1.92 $, as $ \approxFactorCorCls = 1.4371 $.
    \end{itemize}
    Under this setting,~\cref{thm:fair.consensus.clustering} yields a $ ( \approxFactorClsFair + 1.92 ) $-approximation fair consensus clustering algorithm with running time $ O(m^{4}n^{2} + m^{3}\runtimeClsFair{n}) $.

\end{proof}

\section{Missing Details of Streaming Fair Consensus Clustering Algorithm}

\noindent \textbf{Pseudocodes of streaming algorithm}
\begin{algorithm}
\caption{Streaming $1$-Median Fair Consensus, $\sone$}
\label{alg:sone}
\DontPrintSemicolon
\SetKwInOut{KwIn}{Input}
\SetKwInOut{KwOut}{Output}
\SetKwFunction{BuildUF}{BuildClustersFromUF}
\SetKwFunction{FairCand}{\fc}
\SetKwFunction{ArgMin}{ArgMin}
\KwIn{Stream of triples $((u,v), j, b)$ defining pairwise relations for clustering $C_j$; parameters $n,m,\varepsilon \in (0,\tfrac15]$; access to a $\gamma$-close fair clustering subroutine.}
\KwOut{A fair consensus clustering $\m{F}$.}

\BlankLine
\textbf{Sampled indices and memory structures:}\;
Draw $s \gets \lceil \log m \rceil$ indices $J \gets \{j_1,\dots,j_s\} \subseteq [m]$ uniformly at random.\;
Let $\mathsf{M}_1 \leftarrow$ empty store (hash from index $\to$ list of triples).\;
Set $t \gets \lceil \varepsilon^{-2} \log m \rceil$ and draw $K \gets \{k_1,\dots,k_t\} \subseteq [m]$ uniformly at random.\;
Let $\mathsf{M}_2 \leftarrow$ empty store (hash from index $\to$ list of triples).\;

\BlankLine
\textbf{Single-pass streaming phase:}\;
\ForEach{triple $((u,v), j, b)$ arriving in the stream}{
    \If{$j \in J$}{append $((u,v), b)$ to $\mathsf{M}_1[j]$\;}
    \If{$j \in K$}{append $((u,v), b)$ to $\mathsf{M}_2[j]$\;}
}

\BlankLine
\textbf{Post-stream reconstruction:}\;
\ForEach{$j \in J$}{
    initialize a fresh union--find $\mathsf{UF}_j$ on $V$\;
    \ForEach{$((u,v), b) \in \mathsf{M}_1[j]$}{
        \If{$b=0$}{union $u$ and $v$ in $\mathsf{UF}_j$\;}
    }
    $\m{C}_j \gets$ \BuildUF{$\mathsf{UF}_j$}\tcp*{extract clusters from union--find}
}
Let $\mathcal{C}_J \gets \{\m{C}_{j_1},\dots,\m{C}_{j_s}\}$.\;

\BlankLine
\textbf{Generate fair candidate set:}\;
$\widetilde{\m{F}} \gets$ \FairCand{$\mathcal{C}_J$}\tcp*{uses $\gamma$-close fair clustering oracle}

\BlankLine
\textbf{Build $\m{W}$ from $\mathsf{M}_2$:}\;
\ForEach{$k \in K$}{
    initialize a fresh union--find $\mathsf{UF}_k$ on $V$\;
    \ForEach{$((u,v), b) \in \mathsf{M}_2[k]$}{
        \If{$b=0$}{union $u$ and $v$ in $\mathsf{UF}_k$\;}
    }
    $\m{C}_k \gets$ \BuildUF{$\mathsf{UF}_k$}\;
}
$\m{W} \gets \{\m{C}_{k_1},\dots,\m{C}_{k_t}\}$.\;

\BlankLine
\textbf{Select the best fair clustering on the sampled store $2$:}\;
$\displaystyle \m{F} \gets \ArgMin_{\m{F}' \in \widetilde{\m{F}}} \ \sum_{\m{C}_i \in \m{W}} \dist(\m{C}_i, \m{F}')$\;

\Return $\m{F}$\;
\end{algorithm}

\begin{algorithm}
\caption{$\fc(\m{I})$}
\label{alg:streaming-fair}
\KwIn{A set $\m{I}$ of clusterings}
\KwOut{A set $\canset$ of candidate fair clusterings}
\BlankLine
Initialize an empty set $ \candidateSet $\;
    \For{each $ \m{C}_i \in \m{I}  $}{
        $ \m{F}_i \gets $ a $ \approxFactorClsFair $-close fair clustering to $ \concls{i} $ \;
        $ \canset \gets \canset\cup \{\m{F}_i\} $ \;
    }

    \For{each $ \tripleconcls = \{\m{C}_i, \m{C}_j, \m{C}_k\} $ of $ \m{I} $}{
        $ \fairtriple \gets \clsfitting{ \tripleconcls } $\;
        $ \candidateSet \gets \candidateSet\cup \{\fairtriple\} $
    }
    \KwRet{$\canset$}
\end{algorithm}

\noindent\textbf{Space Complexity.}
In $\mathsf{M}_1$, we store information of $\lceil \log m \rceil$ many clusterings, hence we require a space of $O(n \log m)$ in $\mathsf{M}_1$. In $\mathsf{M}_2$ also, we store $O(\log m)$ clusterings; hence, in $\mathsf{M}_2$, too we need $O(n \log m)$ space. Hence, overall, the space complexity is $O(n \log m)$.

\noindent\textbf{Time Complexity.}

In $\mathsf{M}_1$ we store $\log m$ many clusterings. Hence, there are $O(\log^3 m)$ many candidate clusterings in $\canset$. Computing a $ \approxFactorClsFair $-close fair clustering to each sampled clustering and running the procedure $ \clsfitting $ on $ \log^{3} m $ triples of clusterings takes $ \runtimeClsFair{n}\log m + \runtimeFairCorCls{n}\log^{3} m $ time. To find the best clustering in $\canset$, we need to perform $O(\log^3 m)$ many checks. Now, in each check, we need to compute the distance between two clusterings, which takes $O(n^2)$ time, and since there are $O(\log m)$ clusterings in the $\m{W}$, the total time needed per check is $O(n^2 \log m)$. Hence, overall, the time complexity is $O(n^2 \log^4 m)$.

\noindent \textbf{Proof of~\cref{thm:explicit.fair.consensus.clustering.randomized}}
\begin{proof}[Proof of~\cref{thm:explicit.fair.consensus.clustering.randomized}]
    We apply~\cref{lem:fair.consensus.clustering.randomized} and~\cref{cor:approx.fair.correlation.clustering} with $ \beta = \frac{1}{72(\approxFactorClsFair+1)(\approxFactorCorCls+1)} \in (0,1) $, $ \alpha = \frac{1}{6(\approxFactorCorCls+1)} $, $ s=4 $, $ c=3 $, and $ g \geq \max\left(100\times \frac{4}{3} \times (\approxFactorClsFair+1)\frac{216(\approxFactorCorCls+1)}{23}, 100\times 6 \times (\approxFactorClsFair+1)\frac{72(\approxFactorCorCls+1)}{5}\right) $. Note that
    \begin{itemize}
        \item $ 1 + 3(\approxFactorFairCorCls+1)\alpha = 1 + 3(\approxFactorClsFair+1)(\approxFactorCorCls+1)\alpha = \frac{\approxFactorClsFair}{2} + 1.5 $.
        \item $ \approxFactorClsFair + 2 - (\approxFactorClsFair+1)\beta = \approxFactorClsFair + 2 - \frac{\approxFactorClsFair+1}{72(\approxFactorClsFair+1)(\approxFactorCorCls+1)} = \approxFactorClsFair + 2 - \frac{1}{72(\approxFactorCorCls+1)} < \approxFactorClsFair + 1.995 $.
        \item $ \approxFactorClsFair + 2 + (\approxFactorClsFair+1)\frac{\beta(g-s)+s}{g(s-1)} - \frac{2\alpha}{c} = \approxFactorClsFair + 2 + (\approxFactorClsFair+1)\frac{\beta(g-4)+4}{3g} - \frac{2\alpha}{3} = \approxFactorClsFair + 2 + \frac{1}{216(\approxFactorCorCls+1)} - \frac{4\beta(\approxFactorClsFair+1)}{3g} + \frac{4(\approxFactorClsFair+1)}{3g} - \frac{1}{9(\approxFactorCorCls+1)} = \approxFactorClsFair + 2 - \frac{23}{216(\approxFactorCorCls+1)} - \frac{4\beta(\approxFactorClsFair+1)}{3g} + \frac{4(\approxFactorClsFair+1)}{3g} \leq \approxFactorClsFair + 2 - \frac{23}{216(\approxFactorCorCls+1)} \frac{99}{100} < \approxFactorClsFair + 1.96 $, where in the first inequality, we drop the negative term $ -\frac{4\beta(\approxFactorClsFair+1)}{3g} $ and use the choice of $ g $ so that $ \frac{4(\approxFactorCorCls+1)}{3g} \leq \frac{23}{100\times 216 (\approxFactorCorCls+1)} $, and the last inequality is due to $ \approxFactorCorCls = 1.4371 $ .
        \item $ \approxFactorClsFair + 2 + (\approxFactorClsFair + 1 )\frac{\beta(g(2c+s)-sc)+sc)}{g(cs-s-2c)} - 2\left( 1- \frac{1}{s} - \frac{1}{c}  \right)\alpha = \approxFactorClsFair + 2 + (\approxFactorClsFair+1)\frac{\beta(10g - 12)+12}{2g} - \frac{5\alpha}{6} = \approxFactorClsFair + 2 + \frac{5}{72(\approxFactorCorCls+1)} - \frac{6\beta(\approxFactorClsFair+1)}{g} + \frac{6(\approxFactorClsFair+1)}{g} - \frac{5}{36(\approxFactorCorCls+1)} \leq \approxFactorClsFair + 2 - \frac{5}{72(\approxFactorCorCls+1)} - \frac{6\beta(\approxFactorClsFair+1)}{g} + \frac{6(\approxFactorClsFair+1)}{g} \leq \approxFactorClsFair + 2 - \frac{5}{72(\approxFactorCorCls+1)} \frac{99}{100} < \approxFactorClsFair + 1.98 $, where in the first inequality, we drop the negative term $ -\frac{6\beta(\approxFactorClsFair+1)}{g} $ and use the choice of $ g $ so that $ \frac{6(\approxFactorCorCls+1)}{g} \leq \frac{5}{100\times 72 (\approxFactorCorCls+1)} $, and the last inequality is due to $ \approxFactorCorCls = 1.4371 $.
    \end{itemize}
    Under this setting,~\cref{lem:fair.consensus.clustering.randomized} yields a streaming $ (\approxFactorClsFair+1.995) $-approximation fair consensus clustering algorithm using $ O(n\log m ) $ space, with a query time of $ O(n^{2}\log^{4}m + \runtimeClsFair{n}\log^{3}m) $, with probability at least $ 1 - \frac{1}{m} $.
\end{proof}

\section{Details of Algorithm for $k$-median Fair Consensus Clustering}
\begin{proof}[Proof of \cref{thm:explicit.k.fair.consensus.clustering}]
    The algorithm is as follows: for each of the clusterings $\m{C}_i \in \m{I}$ we compute its closest fair clustering $\m{F}_i$. For each triple, $(\m{C}_x, \m{C}_y, \m{C}_z) \in \m{I} \times \m{I} \times \m{I}$ we apply the cluster fitting algorithm (\cref{alg.cluster.fitting}) to get another set of candidates $\m{T}_j$ for $j \in \binom{m}{3}$. Hence, the set of candidates we get is 
    \[
        \canset = \left\{ \m{F}_1, \m{F}_2, \ldots, \m{F}_m, \m{T}_1, \ldots, \m{T}_{\binom{m}{3}}\right\}
    \]
    Now we do a brute-force search among the candidates $\canset$ to find the set of $k$-representatives, $\mathsf{Z} = \{ \m{Z}_1, \ldots, \m{Z}_k \}$. More specifically, for each 
    \[
        \binom{m + \binom{m}{3}}{k} \, \, \text{subset of size $k$,}\, \, \mathsf{Z} = (\m{Z}_1, \ldots, \m{Z}_k) \subseteq \canset 
    \]
    we compute $\objec(\canset, S)$ and return the $k$-tuple for which this objective value is minimum.

    To prove that the set of representatives $\mathsf{Z} = \{ \m{Z}_1, \ldots, \m{Z}_k \}$ returned by the above algorithm gives us $(\gamma + 1.92)$ approximation, we need to prove the following claim.
    \begin{claim}
        There exists a set of $k$-representatives, $\mathsf{Z} = \{ \m{Z}_1, \ldots, \m{Z}_k \} \subseteq \canset$ such that 
    \[
    \objec(\m{I}, \mathsf{Z}) \leq (\gamma + 1.92) \objec(\m{I}, \mathsf{Z}^*)
    \]
    where $\mathsf{Z}^* = \{\m{Z}_1^*, \ldots, \m{Z}_k^*\}$ is the set of optimal representatives.
    \end{claim}
    \begin{proof}
        To facilitate analysis, we introduce the notions of a \emph{super-cluster} and a \emph{super-clustering}.  
A super-cluster is a subset of \(\cset\), while a super-clustering is a collection of such super-clusters that are pairwise disjoint and whose union equals \(\cset\).

Given a set of representatives \(\oset = \{\m{Z}_1, \ldots, \m{Z}_k\}\), we can partition the input clusterings into \(k\) super--clusters \(\{M_1, \ldots, M_k\}\), where each \(M_j\) contains the clusterings that are closest to $\m{S}_j$.
\[
    M_j
      = \bigl\{
            \mathcal{C}_i \in \cset
            \mid
            \dist(\mathcal{C}_i, \mathcal{Z}_j)
               \leq
            \dist(\mathcal{C}_i, \mathcal{Z}_\ell)
            \quad
            \forall\, \mathcal{Z}_\ell \in \oset
        \bigr\}.
\]

        Let $M_j^*$ be a super-cluster corresponding to the optimal representative $\m{Z}_j^*$. By definition, we have,
        \[
            \objec(\m{I}, \m{Z}^*) = \sum_{j = 1}^k \objec(M_j^*, \m{Z}_j^*)
        \]
        Thus, if we can get an $(\gamma + 1.92)$ approximation algorithm for each of the super-clusters $M_j^*$, that is if we can find a representative $\m{Z}_j$ for all $j \in [1,k]$ such that
        \[
            \objec(M_j^*, \m{Z}_j^*) \leq \objec(M_j^*, \m{Z}_j)
        \]
        Then we get a set of representatives $\m{Z} = \{ \m{Z}_1, \ldots, \m{Z}_k \}$ such that 
        \begin{align*}
            \objec(\m{I}, \m{Z}^*) &= \sum_{j = 1}^k \objec(M_j^*, \m{Z}_j^*) \n \\
            &\leq (\gamma + 1.92) \sum_{j = 1}^k \objec(M_j^*, \m{Z}_j) \n \\
            &= (\gamma + 1.92) \objec(\m{I}, \m{Z})
        \end{align*}

        Let, $M_j^* = \{ \m{C}_1, \ldots, \m{C}_t \}$. Thus, the candidate set of fair clusterings for $M_j^*$ would be
        \[
            \canset_j = \left\{ \m{F}_{j_1},\ldots, \m{F}_{j_t},\m{T}_{j_1},\ldots, \m{T}_{j_{\binom{t}{3}}} \right\}
        \]
        where $\m{F}_{j_k}$ is the $\gamma$-close fair clustering to $\m{C}_{j_k}$ and $\m{T}_{j_\ell}$ are the clusterings achieved by the cluster fitting algorithm on each set of triples $(\m{C}_x, \m{C}_y, \m{C}_z) \in M_j^* \times M_j^* \times M_j^*$.

        From \cref{sec:fair_consensus_clustering}, we know that there exists $\m{Z}_j \in \canset_j$ such that
        \[
            \objec(M_j^*, \m{Z}_j^*) \leq (\gamma + 1.92) \objec(M_j^*, \m{Z}_j)
        \]
        Since for two optimal clusters $M_j^*$ and $M_\ell^*$ we have 
        \[\
        \canset_j \cap \canset_\ell = \emptyset
        \]
        And, we also have
        \[
            \bigcup_{j = 1}^k \canset_j = \canset.
        \]
        hence, we get that there exists a set of $k$-representatives  $\m{Z} = \{ \m{Z}_1, \ldots, \m{Z}_k \}$ such that
        \begin{align*}
            \{\m{Z}_1, \ldots, \m{Z}_k\} &\subseteq \bigcup_{j = 1}^k \canset_j =\canset \n
        \end{align*}
        This completes the proof.
    \end{proof}
    Now, let us discuss the time complexity of our brute force algorithm. Our first step is to find a set of candidates $\canset$. The set $\canset$, consists of $\gamma$-close fair clusterings to each of the input clusterings $\m{C}_i \in \m{I}$ and a set of clusterings $\{\m{T}_1, \ldots, \m{T}_{m \choose 3} \}$ which we get using our $\clsfitting$ algorithm. The running time to find a $\gamma$-close fair clustering is $\runtimeClsFair{n}$, since we find $\gamma$-close fair clustering for each input clustering $\m{C}_i \in \m{I}$, hence the overall runtime is $O(m \runtimeClsFair{n})$. The runtime of $\clsfitting$ is $O(n^2 + \runtimeFairCorCls{n})$, since we apply the $\clsfitting$ algorithm to every triple $\{ \m{C}_x, \m{C}_y, \m{C}_z \}$ of $\m{I}$, hence we use the $\clsfitting$ algorithm $O(m^3)$ times. Thus, overall, the time complexity in using the $\clsfitting$ algorithm is $O(m^3 (n^2 + \runtimeFairCorCls{n}))$
    
    Now, there are $O(m^3)$ elements in $\canset$ and we check our objective $\objec(\m{I}, S)$ for each of the $\binom{m^3}{k}$ many candidates in $\canset$. Thus, we check the objective $O(m^{3k})$ times. In each check, we need to compute the distance between two clusterings, which takes $O(n^2)$ time, and since there are $m$ input clusterings, the total time needed per check is $O(mn^2)$.
    Hence, overall, the time complexity is $O(m^{3k + 1}n^2)$.

    Thus, by combining the time needed to find the candidates and the time needed to find the best fair clustering among the candidates, we get that the total time complexity is

    \[
        O(m \runtimeClsFair{n} + m^3 (n^2 + \runtimeFairCorCls{n}) + m^{3k + 1}n^2).
    \]
\end{proof}

\section{Analysis of Streaming $k$-median Fair Consensus Clustering Algorithm}

We fix an optimal $ k $--median clusterings $ \fairmediank = \{ \fairmediank{1}, \fairmediank{2}, \dots, \fairmediank{k}\} $ of $ \inpconclss $, and let $ \{\optclsk{1}, \optclsk{2}, \dots, \optclsk{k} \} $ be the corresponding induced clusters. For each $ i\in [k] $, we denote by $ \optobjk{i} = \sum_{\concls{j}\in \optclsk{i}} \dist(\concls{j}, \fairmediank{i}) $ the distance from all clusterings from $ \optclsk{i} $ to  $ \fairmediank{i} $. By definition, $ \optconval = \sum_{i=1}^{k} \optobjk{i} $.

For any $ \sampDist \in \{ \frac{1}{2}, \frac{1+\rate}{2}, \dots, n^{2}  \} $ and $ \sampRate\in \{ 1, \frac{1}{1+\rate}, \dots, \frac{1}{m} \} $, we consider the set of candidates constructed from $ \sampSet{\sampDist, \sampRate} $, that is,
{ \small
    \begin{align}
    \candidate{\sampSet{\sampDist, \sampRate}} = &\{ \fairconcls | \concls\in \sampSet{\sampDist, \sampRate}, \fairconcls \text{ is a } \approxFactorClsFair\text{-close fair clustering to } \concls \}. \nonumber\\
    \>\>\> \cup &\{ \fairtriple| \fairtriple = \clsfitting{ \tripleconcls } \text{, for each } \tripleconcls\subseteq \sampSet{\sampDist, \sampRate} \text{ such that } \card{\tripleconcls} = 3 \}. \nonumber
\end{align} 
}
The following lemma demonstrates that for any $ i \in [k] $ such that $ \optobjk{i}\geq \frac{\optKFac \optconval}{k} $, for an appropriate choice of $ \sampDist $ and $ \sampRate $, the set $ \candidate{\sampSet{\sampDist, \sampRate}} $ contains a good approximation for an optimal fair median clustering for $ \optclsk{i} $.
\begin{lemma}\label{lem:either.input.or.triple}
    Take an arbitrarily small constant $ \rate < 0.1 $. For any constant $ \optKFac $ and for any $ i\in [k] $ with $ \optobjk{i}\geq \frac{\optKFac \optconval}{k} $, then there exsists $ \sampDist\in \{\frac{1}{2},\frac{1+\rate}{2}, \dots, n^{2}\} $ and $ \sampRate \in \{1, \frac{1}{1+\rate}, \dots, \frac{1}{m} \} $, such that with high probability $ \sampSet{\sampDist, \sampRate} $ is not set to empty in~\ref{enu.remove.large.sample.set}, and there is a clustering $ \fairconcls \in \candidate{\sampSet{\sampDist, \sampRate}} $ satisfying $ \obj{\optclsk{i}, \fairconcls} \leq r\optobjk{i} $, where
    \begin{align}
        r = \max\bigl( &2+\approxFactorClsFair - ( \beta - \sampDF )(1+\approxFactorClsFair), \nonumber \\ 
                  &2+ \sampDF + \approxFactorClsFair(1+\beta+\sampDF) - \dfrac{\alpha\beta}{2(\beta+1)} + \dfrac{2\beta^{2}}{1-\beta} + o_{m}(1), \nonumber \\
                  &(1 + 3(\approxFactorClsFair+1)(\approxFactorCorCls+1)\alpha)\bigr). \nonumber
    \end{align}
\end{lemma}
Prior to proving \cref{lem:monotone.faraway.sampling}, we provide definitions and observations that will be frequently used in the proof. Let $ d_{i} = \frac{\optobjk{i}}{\card{\optclsk{i}}} $. We partition the set $ \optclsk{i} $ into three sets of clusterings based on their distance to $ \fairmediank{i} $. In particular, we define
\begin{itemize}
    \item $ \near{i} = \{ \concls\in \optclsk{i}: \dist(\concls, \fairmediank{i})\leq (1-\beta) d_{i} \} $.
    \item $ \midcl{i} = \{ \concls\in \optclsk{i}: (1-\beta)d_{i} < \dist(\concls, \fairmediank{i})\leq (1+\beta)d_{i} \} $.
    \item $ \far{i} = \{ \concls\in \optclsk{i}: \dist(\concls, \fairmediank{i}) > (1+\beta)d_{i} \} $.
\end{itemize}
We further partition $ \midcl{i} $ into two subsets as follows. For each clustering $ \concls \in \optclsk{i} $, let $ \unalignedSet{\concls} $ be the set of pairs of points that are clustered differently between $ \concls $ and $ \fairmediank{i} $. By this definition, $ \dist(\concls, \fairmediank{i}) = \card{\unalignedSet{\concls}} $ and for any two clusterings $ \concls, \concls' \in \optclsk{i} $, $ \dist(\concls, \concls') = \card{\unalignedSet{\concls}} + \card{\unalignedSet{\concls'}} - 2\card{\unalignedSet{\concls}\cap \unalignedSet{\concls'}} $. We are concerned with a ball centered at a clustering $ \concls \in \midcl{i} $, which is $ S(\concls) = \{ \concls'\in \midcl{i} | \card{\unalignedSet{\concls} \cap \unalignedSet{\concls'}} \geq \alpha d_{i} \} $. We define $ \middense{i} = \left\{\concls| \concls\in \midcl{i}, \card{S(\concls)}\geq \frac{\card{\midcl{i}}}{4}\right\} $, and $ \midsparse{i} = \midcl{i}\setminus \middense{i} $.

\begin{observation}\label{obs:far.set.bound}
    For any $ i\in [k] $, $ \card{\far{i}} \leq \frac{\card{\optclsk{i}}}{1+\beta} $.
\end{observation}
\begin{proof}
    Since $ \optobjk{i} =  \sum_{\concls \in \optclsk{i}}\dist(\concls, \fairmediank{i}) \geq \sum_{\concls \in \far{i}} \dist(\concls, \fairmediank{i}) \geq \card{\far{i}}(1+\beta)d_{i} $, it follows that $ \card{\far{i}} \leq \frac{\optobjk{i}}{(1+\beta)d_{i}} = \frac{\card{\optclsk{i}}}{1+\beta} $.
\end{proof}
\begin{observation}\label{obs.sample.set.that.survives}
    For any $ i\in [k] $ with $ \optobjk{i}\geq \frac{\optKFac \optconval}{k} $, consider the set $ \sampSet{\sampDist, \sampRate} $ with $ \sampDist \leq d_{i} \leq (1+\rate)\sampDist $ and $ \frac{\sampRate}{1+\rate} \leq \min\left(1,\frac{2\log^{2} m}{\card{\optclsk{i}}}\right)\leq \sampRate $. Then, with high probability, $ \sampSet{\sampDist, \sampRate} $ is not modified after~\ref{enu.remove.large.sample.set}.
\end{observation}
\begin{proof}
    It suffices to show that with high probability, the size of $ \sampSet{\sampDist,\sampRate} $ is at most $ \card{\sampSet{\sampDist, \sampRate}} < \polylog{m} $. We count the number of clusterings sampled from each set $ \optclsk{j} $.

    By the assumption, we have $ \sampRate \leq \frac{4\log^{2} m}{\card{\optclsk{i}}} $. Hence, with high probability, we sampled at most $ 10 \log^{2} m $ clusterings from $ \optclsk{i} $.

    Consider other sets $ \optclsk{j}\neq \optclsk{i} $. After~\ref{enu.add.to.sample.set}, any two clusterings $ \concls, \concls' $ in $ \optclsk{k} $ are of distance at least $ \sampDF\sampDist $ away from each other. Thus, by a triangle inequality, there is at most one clustering sampled from $ \optclsk{j} $ that is within distance at most $ \frac{\sampDF\sampDist}{2} $ away from $ \fairmediank{j} $. Applying the same argument from~\cref{obs:far.set.bound}, the number of clusterings in $ \optclsk{j} $ that are of distance more than $ \frac{\sampDF\sampDist}{2} $ away from $ \fairmediank{j} $ is at most $ \frac{2\card{\optobjk{j}}}{\sampDF \sampDist} $. It follows that, overall, the number of such clusterings from all $ \optclsk{j}\neq \optclsk{i} $ is at most
    \begin{align}
        \sum_{j\in [k]\setminus\{i\}} \frac{2\optobjk{j}}{\sampDF\sampDist} \leq   \dfrac{2\optconval}{\sampDF\sampDist} &\leq \dfrac{2k(1+\rate)\optobjk{i}}{\sampDF \optKFac d_{i}} \leq \dfrac{4k\card{\optclsk{i}}}{\sampDF \optKFac}, \nonumber
    \end{align}
    where the second inequality follows from the assumption that $ \optobjk{i}\geq \frac{\optKFac \optconval}{k} $ and $ d_{i}\leq (1+\rate)\sampDist $.

    As $ \sampRate \leq \frac{4\log^{2} m}{\card{\optclsk{i}}} $, with high probability, we sample at most $ \frac{100k\log^{2} m}{\card{\optclsk{i}}} $ clusterings from $ \inpconclss\setminus \optclsk{i} $. Hence, with high probability, the total number of clusterings in $ \sampSet{\sampDist, \sampRate} $ is at most $ k\log^{3} m $, and therefore, $ \sampSet{\sampDist, \sampRate} $ is not modified after~\ref{enu.remove.large.sample.set}.
\end{proof}

The following observation illustrates that when the set $ \near{i} $ of clusterings close to the optimal is small, then there is an upper bound on the distance between any clustering in $ \concls\in \middense{i} $ to all other clusterings in $ \optclsk{i} $. This later allows us to use fair clusterings close $ \concls $ as a good approximation.
\begin{observation}\label{obs:a.good.middle.clustering}
    Suppose that $ \card{\near{i}} \leq\frac{\card{\optclsk{i}\setminus \far{i}}}{\log m} $. Consider a clustering  $ \concls\in \middense{i} $. Let $ \alpha>0, 1>\beta>0 $ be constants such that $ \alpha/2 > 2\beta(1+\beta)/(1-\beta) $. Then
    \begin{align}
        \obj{\optclsk{i}, \concls} \leq \left(2-\dfrac{\alpha\beta}{2(1+\beta)} +\dfrac{2\beta^{2}}{1-\beta} + o_{m}(1)\right)\optobjk{i}. \nonumber
    \end{align}
\end{observation}
\begin{proof}
    We upper bound the distance between $ \concls $ and any other clustering $ \concls{j}\in \midcl{i} $ as follows.
    \begin{itemize}
        \item If $ \concls{j}\in \far{i} $, then $ \dist( \concls{j}, \concls ) \leq 2 \card{\unalignedSet{\concls{j}}}$.
        \item If $ \concls{j}\in \near{i} $, then $ \dist( \concls{j}, \concls ) \leq \card{\unalignedSet{\concls{j}}} + (1+\beta)d_{i} $.
        \item If $ \concls{j}\in \midcl{i} $, then $ \dist(\concls{j}, \concls) \leq 2\card{\unalignedSet{\concls{j}}} +\dfrac{2\beta(1+\beta)}{1-\beta}d_{i}  - 2\card{\unalignedSet{\concls{j}}\cap \unalignedSet{\concls}} $.
    \end{itemize}
    To see this, recall that from~\cref{eq.dist.unalignedSet.relation}, $ \dist(\concls{j}, \concls) = \card{\unalignedSet{\concls{j}}}+\card{\unalignedSet{\concls}} - 2\card{\unalignedSet{\concls{j}}\cap \unalignedSet{\concls}} $. Now, for $ \concls{j}\in \far{i} $, $ \card{\unalignedSet{\concls{j}}} \geq \card{\unalignedSet{\concls}} $, the first item follows. For the second item, $ \card{\unalignedSet{\concls}}\leq (1+\beta)d_{i} $ by definition. For the last item, since $ (1-\beta)d_{i}\leq \card{\unalignedSet{\concls{j}}}\leq (1+\beta)d_{i} $, it follows that $ \card{\unalignedSet{\concls}} \leq (1+\beta)d_{i} \leq (1+\frac{2\beta}{1-\beta})\card{\unalignedSet{\concls{j}}} \leq \card{\unalignedSet{\concls{j}}} +  \frac{2\beta(1+\beta)}{1-\beta}d_{i} $. 

    Now, note that, as $ \concls{i}\in \middense{i} $, hence, $ \card{S(\concls)} \geq \frac{\card{\midcl{i}}}{4} $. Thus
    \begin{align}
        &\>\>\>\>\>\obj{\optclsk{i}, \concls} \nonumber \\
        &=\sum_{\concls{j} \in \optclsk{i}} \dist(\concls, \concls{j}) = \sum_{\concls{j}\in \optclsk{i}} \left(\card{\unalignedSet{\concls}} + \card{\unalignedSet{\concls_{j}}} - 2\card{\unalignedSet{\concls}\cap \unalignedSet{\concls_{j}}}\right) \nonumber\\
        &\leq 2\optobjk{i} + (1+\beta)d_{i} \card{\near{i}} +   \dfrac{2\beta(1+\beta)}{1-\beta}d_{i}\card{\midcl{i}} - \alpha d_{i} \dfrac{\card{\midcl{i}}}{2} \nonumber\\
        &\leq 2\optobjk{i} + (1+\beta)d_{i} \dfrac{\card{\optclsk{i}\setminus \far{i}}}{\log m} \nonumber \\ &\> - \left(\dfrac{\alpha}{2} -\dfrac{2\beta(1+\beta)}{1-\beta}\right)d_{i}\dfrac{(\log m-1)\card{\optclsk{i}\setminus \far{i}}}{\log m} \nonumber\\
        &\leq 2\optobjk{i} - \left(\dfrac{\alpha}{2} -\dfrac{2\beta(1+\beta)}{1-\beta} + o_{m}(1) \right) \dfrac{\beta d_{i}\card{\optclsk{i}}}{1+\beta} \nonumber \\
        &= \left(2-\dfrac{\alpha\beta}{2(1+\beta)} +\dfrac{2\beta^{2}}{1-\beta} + o_{m}(1)\right)\optobjk{i}, \nonumber
    \end{align}
    where the second inequality follows from $ \card{\near{i}} \leq\dfrac{\card{\optclsk{i}\setminus \far{i}}}{\log m} $, $ \alpha/2 > 2\beta(1+\beta)/(1-\beta) $, and $ \card{\midcl{i}}\geq (\log m -1)\card{\optclsk{i}\setminus \far{i}} $. The last inequality follows from $ \card{\optclsk{i}\setminus \far{i}}\geq \beta\card{\optclsk{i}}/(1+\beta) $, which is implied from~\cref{obs:far.set.bound}, $ \card{\far{i}} \leq \card{\optclsk{i}}/(1+\beta) $.
\end{proof}

\begin{proof}[Proof of~\cref{lem:either.input.or.triple}]
We proceed by cases.

\textbf{Case 1.} $ \card{\near{i}}\geq\frac{\card{\optclsk{i}\setminus \far{i}}}{\log m } $.

From~\cref{obs:far.set.bound}, we have that $ \card{\optclsk{i}\setminus \far{i}} \geq \frac{\beta\card{\optclsk{i}}}{1+\beta} $. Thus, $ \card{\near{i}}\geq \frac{\beta\card{\optclsk{i}}}{(1+\beta)\log m } $. We consider the set $ \sampSet{\sampDist, \sampRate} $ with $ \sampDist \leq d_{i} \leq (1+\rate)\sampDist $, and $ \frac{\sampRate}{\rate} <  \min\left(1,\frac{2\log^{2} m}{\card{\optclsk{i}}}\right) \leq \sampRate $. By~\cref{obs.sample.set.that.survives}, $ \sampSet{\sampDist, \sampRate} $ is not set to empty in~\ref{enu.remove.large.sample.set} with high probability. Since $ \card{\near{i}} \geq \frac{\beta\card{\optclsk{i}}}{(1+\beta)\log m} $, by a Chernoff bound, with high probability, we sample at least $   \frac{\beta\log m}{10(1+\beta)} $ clusterings from $ \near{i} $. Let $ \concls $ be such a clustering and let $ \fairconcls $ be a $ \approxFactorClsFair $-close fair clustering to $ \concls $. If $ \concls \in \sampSet{\sampDist, \sampRate} $, then $ \fairconcls $ guarantees us a desired approximation, as by a triangle inequality, $ \dist(\fairconcls, \fairmediank{i}) \leq (1+\approxFactorClsFair)\dist(\concls, \fairmediank{i}) \leq (1+\approxFactorClsFair)(1-\beta)d_{i} $, and hence, again by a triangle inequality,
\begin{align}
    \obj{\optclsk{i}, \fairconcls} &\leq \sum_{\concls{j} \in \optclsk{i}} \left(\dist(\concls{j}, \fairmediank{i}) + \dist(\fairmediank{i}, \fairconcls)\right) \nonumber\\
                                   &\leq \optobjk{i} + (1+\approxFactorClsFair)(1-\beta)d_{i}\card{\optclsk{i}} \nonumber\\
                                   &=(2+\approxFactorClsFair - \beta(1+\approxFactorClsFair))\optobjk{i}.\nonumber
\end{align}
Otherwise, if $ \concls \notin \sampSet{\sampDist, \sampRate} $, then there must be a clustering $ \concls' \in \sampSet{\sampDist, \sampRate} $ such that $ \dist(\concls, \concls') < \sampDF\sampDist $. By a triangle inequality, we have that $ \dist(\concls', \fairmediank{i}) \leq (1-\beta)d_{i} + \sampDF\sampDist \leq (1-\beta+\sampDF)d_{i} $. Let $ \fairconcls' $ be a $ \approxFactorClsFair $-close fair clustering to $ \concls' $. By a triangle inequality, we have that $ \dist(\fairconcls', \fairmediank{i}) \leq (1+\approxFactorClsFair)(1-\beta+\sampDF)d_{i} $. Using a similar argument as above, we have $ \obj{\optclsk{i}, \fairconcls'} \leq (2+\approxFactorClsFair - (\beta-\sampDF)(1+\approxFactorClsFair))\optobjk{i} $.

\textbf{Case 2.} $ \card{\near{i}} < \frac{\card{\optclsk{i}\setminus \far{i}}}{\log m } $ and $ \card{\middense{i}} \geq \frac{\card{\midcl{i}}}{5} $.

Applying~\cref{obs:far.set.bound}, we have $ \card{\midcl{i}} \geq \frac{(\log m -1)\card{\optclsk{i}\setminus \far{i}}}{\log m }\geq \frac{(\log m -1)\beta \card{\optclsk{i}}}{\log m (1+\beta)} $. It follows that $ \card{\middense{i}} \geq \frac{(\log m -1)\beta \card{\optclsk{i}}}{5\log m (\beta+1)} $.

We consider the set $ \sampSet{\sampDist, \sampRate} $ with $ \sampDist \leq d_{i} \leq (1+\rate)\sampDist $, and $ \frac{\sampRate}{1+\rate} <  \min\left(1,\frac{2\log^{2}m}{\card{\optclsk{i}}}\right) \leq \sampRate $. By~\cref{obs.sample.set.that.survives}, $ \sampSet{\sampDist, \sampRate} $ is not set to empty in~\ref{enu.remove.large.sample.set} with high probability. Note that, with high probability, we sample at least $ \frac{\log m }{10} $ clusterings from $ \middense{i} $. Let $ \concls $ be such a clustering and let $ \fairconcls $ be its corresponding $ \approxFactorClsFair $-close fair clustering. If $ \concls \in \sampSet{\sampDist, \sampRate} $, we show that $ \fairconcls $ gives us a desired approximation. Indeed,by~\cref{obs:a.good.middle.clustering}
\begin{align}
    &\obj{\optclsk{i}, \concls} \leq \left(2 - \dfrac{\alpha\beta}{2(\beta+1)} + \dfrac{2\beta^{2}}{1-\beta} + o_{m}(1)\right)\optobjk{i}. \nonumber
\end{align}

It follows that
\begin{align}
    &\obj{\optclsk{i}, \fairconcls} \leq \obj{\optclsk{i}, \concls} + \approxFactorClsFair\dist(\concls, \fairmediank{i})\card{\optclsk{i}}\nonumber\\ 
    &\leq \left(2+ \approxFactorClsFair(1+\beta) - \dfrac{\alpha\beta}{2(\beta+1)} + \dfrac{2\beta^{2}}{1-\beta} + o_{m}(1) \right)\optobjk{i}. \nonumber
\end{align}

On the other hand, if $ \concls\notin \sampSet{\sampDist, \sampRate} $, then after~\ref{enu.add.to.sample.set}, there must be a clustering $ \concls' \in \sampSet{\sampDist, \sampRate} $ such that $ \dist(\concls, \concls') < \sampDF\sampDist $. By a triangle inequality, we have that $ \dist(\concls', \fairmediank{i}) \leq (1+\beta)d_{i} + \sampDF\sampDist \leq (1+\beta+\sampDF)d_{i} $. Let $ \fairconcls' $ be a $ \approxFactorClsFair $-close fair clustering to $ \concls' $. Then by a triangle inequality, we have $ \dist(\fairconcls', \concls')\leq \approxFactorClsFair \dist(\concls', \fairmediank{i}) \leq \approxFactorClsFair(1+\beta+\sampDF)d_{i} $. Using a similar argument as above, we have
\begin{align}
    &\obj{\optclsk{i}, \fairconcls'} \leq \obj{\optclsk{i}, \concls} + \card{\optclsk{i}}\left(\dist(\concls, \concls') + \dist(\concls', \fairconcls')\right) \nonumber\\
    &\leq \left(2+ \sampDF + \approxFactorClsFair(1+\beta+\sampDF) - \dfrac{\alpha\beta}{2(\beta+1)} + \dfrac{2\beta^{2}}{1-\beta} + o_{m}(1)\right)\optobjk{i}. \nonumber
\end{align}

\textbf{Case 3}. $ \card{\near{i}} < \frac{\card{\optclsk{i}\setminus \far{i}}}{\log m } $ and $ \card{\middense{i}} < \frac{\card{\midcl{i}}}{5} $.

Similarly, we have $ \card{\midcl{i}} \geq \frac{(\log m -1)\beta \card{\optclsk{i}}}{\log m (1+\beta)} $, and $ \card{\midsparse{i}}\geq \frac{4\card{\midcl{i}}}{5} \geq \frac{4(\log m -1)\beta \card{\optclsk{i}}}{5\log m (1+\beta)} $.

We also consider the set $ \sampSet{\sampDist, \sampRate} $ with $ \sampDist \leq d_{i} \leq (1+\rate)\sampDist $, and $ \frac{\sampRate}{1+\rate} <  \min\left(1,\frac{2\log^{2} m}{\card{\optclsk{i}}}\right) \leq \sampRate $, and recall that $ \sampSet{\sampDist, \sampRate} $ is not set to empty in~\ref{enu.remove.large.sample.set} with high probability. 

We show that with high probability, $ \sampSet{\sampDist, \sampRate} $ contains a set of three clusterings, without loss of generality, denoted by $ \tripleconcls=\{\concls{1}, \concls{2}, \concls{3} \} $, such that $ \card{\unalignedSet{r}\cap \unalignedSet{s}}\leq \alpha d_{i} $, for all $ 1\leq r < s \leq 3 $.

To this end, note that with high probability, we sample at least $ \log m  $ clusterings from $ \midsparse{i} $. Let $ \concls{1} $ be such a clustering. Observe that $ \card{\midsparse{i} \setminus S(\concls{1})} \geq \frac{4 \card{\midcl{i}}}{5} - \frac{\card{\midcl{i}}}{4} = \frac{11(\log m -1)\beta\card{\optclsk{i}}}{20\log m (1+\beta)} $, hence, given that $ \concls{1} $ is sampled, with high probability, we sample at least $ \log m  $ clusterings from $ \midsparse{i}\setminus S(\concls{1}) $. Let $ \concls{2} $ be such a clustering. Similarly, $ \card{\midsparse{i}\setminus (S(\concls{1})\cup S(\concls{2}))} \geq \frac{11(\log m -1)\beta\card{\optclsk{i}}}{20\log m (1+\beta)} - \frac{\card{\midcl{i}}}{4} = \frac{3(\log m -1)\beta\card{\optclsk{i}}}{10\log m (1+\beta)} $, and thus, given that $ \concls{1}, \concls{2} $ are sampled, with high probability, we sample at least $ \log m  $ clusterings from $ \midsparse{i}\setminus (S(\concls{1})\cup S(\concls{2})) $. Let $ \concls{3} $ be such a clustering. By the construction, we have that $ \card{\unalignedSet{r}\cap \unalignedSet{s}}\leq \alpha d_{i} $, for all $ 1\leq r < s \leq 3 $. Applying the procedure $ \clsfitting $ on these three clusterings yields a fair clustering $ \fairtriple $ where by~\cref{lem:cluster.fitting.helps}, $ \obj{\optclsk{i}, \fairtriple} \leq (1 + 3(\approxFactorClsFair+1)(\approxFactorCorCls+1)\alpha)\optobjk{i} $.

\end{proof}

\begin{proof}[Proof of~\cref{thm:streaming.main.theorem}]
    We argue that with high probability, our algorithm has space complexity of $ O(k^{2}n\polylog(mn)) $, update time of $ O((km)^{O(1)}n^{2}\log n ) $, and query time of $ O((k\log(mn))^{O(k)}n^{2} + k^{3}\log^{12} m \log^{3} n\runtimeFairCorCls{n}) $.

    \textbf{Space Complexity}: In \ref{StepOneA}, we keep only sets $ \sampSet{\sampDist, \sampRate} $ with $ \card{\sampSet{\sampDist, \sampRate}} \leq k\log^{3} m $. There are $ O(\log n ) $ different choices of $ \sampDist $ and $ O(\log m ) $ different choices of $ \sampRate $. Therefore, the sampled set $ \sampSet $ contains $ O(k\log^{4} m\log n ) $ clusterings. Note that a fair clustering $ \fairtriple $ produced from a triple of clusterings can be computed and evaluated on the fly, and thus, does not need to be stored. In \ref{StepOneB}, we store a set $ \mfsSet $ of $ O(k\log m ) $ clusterings. In \ref{StepOneB}, by~\cref{lem:monotone.faraway.sampling}, the sampled set $ \mfsSet $ contains at most $ O(k^{2}\log k \log(km)) $ input clusterings with high probability. Additionally, as the size of the implicit candidate set $ \candidateSet $ is bounded by $ O(m^{3}) $, using~\cref{lem:streaming.coreset}, the size of the coreset $ (\coreset, w) $ constructed in \ref{StepOneC} is $ O(k^{2}\log^{2} m) $. Each clustering is stored using $ n^{2}\log n  $ bits. Hence, the total space of our algorithm is $ O(k^{2}n \polylog( mn )) $.

    \textbf{Update time}: The distance between any two clusterings can be computed in $ O(n^{2}) $. As $ \card{\sampSet} = O(k\log^{4} m\log n ) $, the update time of \ref{StepOneA} is $ O(kn^{2}\log^{4} m\log n ) $. In \ref{StepOneB}, by~\cref{lem:monotone.faraway.sampling}, the update time is $ O(k^{2}\log k \log(km)) $. In \ref{StepOneC}, by~\cref{lem:streaming.coreset}, the update time is $ O((k\log m )^{O(1)}) $. Thus, the overall update time is $ O((km)^{O(1)}n^{2}\log n ) $.

    \textbf{Query time}: Since $ \card{\sampSet} = O(k\log^{4} m\log n ) $, there are at most $ \card{\sampSet}^{3} = k^{3}\log^{12} m\log^{3}n$ different choices for $ \fairtriple $. Thus, $ \card{\rcandidateSet} = k^{O(1)}\polylog(mn) $. Each clustering $ \fairtriple $ is computed using the procedure $ \clsfitting $ in time $ O(n^{2}+\runtimeFairCorCls{n}) $ time. The total number of candidate $ k $--median is $ O(k\log(mn))^{O(k)} $. Evaluating the objective of a candidate $ k $--median using the coreset takes time $ O(k^{3}n^{2}\log^{2} m) $. Therefore, the overall query time is $ O((k\log(mn))^{O(k)}n^{2}+ k^{3}\runtimeFairCorCls{n}\log^{12} m\log^{3}n) $.

    It remains to analyze the approximation ratio of our algorithm.

    Let $ A = \{i| \optobjk{i} \leq \frac{\optKFac \optconval}{k}\} $. Recall that $ \mfsSet $ is a set of sampled clustering produced in the monotone faraway sampling step such that each $ i\in A $, there exists a clustering $ \concls{i}'\in \mfsSet $ satisfying $ \obj{\optclsk{i}, \concls{i}'} \leq 5\optobjk{i} + \mfsrho \frac{\optconval}{k}\leq \frac{5\optKFac\optconval}{k} + \mfsrho\frac{\optconval}{k} $. Let $ \fairconcls{i}' $ be a $ \approxFactorClsFair $-close fair clustering to $ \concls{i}' $ added to the candidate set $ \rcandidateSet $. By triangle inequalities, we have $ \dist( \fairconcls{i}', \concls{i}' ) \leq \approxFactorClsFair\dist(\fairmediank{i}, \concls{i}') \leq \approxFactorClsFair (\dist(\fairmediank{i}, \concls{j}) + \dist(\concls{j}, \concls{i}')) $, for each $ \concls{j}\in \optclsk{i} $. It follows that
    \begin{align}
        \obj{\optclsk{i}, \fairconcls{i}'} &\leq \obj{\optclsk{i}, \concls{i}'} +\sum_{\concls{j}\in \optclsk{i}} \dist(\fairconcls{i}', \concls{i}') \nonumber \\
                                           &\leq (1+\approxFactorClsFair)\obj{\optclsk{i}, \concls{i}'} + \approxFactorClsFair \optobjk{i} \nonumber \\
                                           &\leq (5 + 6\approxFactorClsFair)\optobjk{i} + (1+\approxFactorClsFair)\mfsrho\dfrac{\optconval}{k}. \nonumber
    \end{align}
    Taking the sum over all $ i\in A $,
    \begin{align}
        \sum_{i\in A}\obj{\optclsk{i}, \fairconcls{i}'} &\leq (5 + 6\approxFactorClsFair)\sum_{i\in A}\optobjk{i} + (1+\approxFactorClsFair)\mfsrho\optconval \nonumber \\
                                                                &\leq (5 + 6\approxFactorClsFair)\optKFac\optconval + (1+\approxFactorClsFair)\mfsrho\optconval, \nonumber
    \end{align}
    where the last inequality follows from $ \optobjk{i}\leq \frac{\optKFac\optconval}{k} $. Since $ \optKFac $ and $ \mfsrho $ are arbitrarily small positive constants, we set $ \optKFac = \frac{\infsmall}{2} ((5+6\approxFactorClsFair))^{-1},\ \mfsrho = \frac{\infsmall}{2} ((1+\approxFactorClsFair))^{-1} $, and thus, $ \sum_{i\in A}\obj{\optclsk{i}, \fairconcls{i}'} \leq \infsmall\optconval $, where $ \infsmall>0 $ is an arbitrarily small constant.

    Following~\cref{lem:either.input.or.triple}, for each $ i\in [k]\setminus A $, with high probability, in \ref{StepOneA}, we sample a set $ \sampSet{\sampDist, \sampRate} $ such that there is a fair clustering $ \fairconcls\in \candidate{\sampSet{\sampDist, \sampRate}} $ satisfying $ \obj{\optclsk{i}, \fairconcls} \leq r\optobjk{i} $, where $ \candidate{\sampSet{\sampDist, \sampRate}} $ is the set of candidates constructed from $ \sampSet{\sampDist, \sampRate} $. In \textbf{Step 2}, we construct $ \candidate{\sampSet{\sampDist, \sampRate}} $ and include it in $ \rcandidateSet $. Thus, $ \rcandidateSet $ contains $ k $ fair clusterings with total objective at most $ r\optconval + \infsmall\optconval = (r+\infsmall)\optconval $. Since we evaluate this using the $ (k,\errCoreset) $--coreset $ (\coreset,w) $, the total objective is $ (1+\errCoreset)(r+\infsmall)\optconval $.

\end{proof}

\begin{proof}[Proof of~\cref{thm:explicit.streaming.main.theorem}]
    We apply~\cref{thm:streaming.main.theorem} and~\cref{cor:approx.fair.correlation.clustering} with $ \sampDF = \frac{1}{10^{6}(\approxFactorClsFair+1)} $, $ \beta = 0.015 $, and $ \alpha = \frac{8\beta(\beta+1)}{1-\beta}\approx 0.124 $, $ \errCoreset = 0.0000001 $, and $ \infsmall = 0.0000001 $. Note that
    \begin{itemize}
        \item $ 2+\approxFactorClsFair - ( \beta - \sampDF )(1+\approxFactorClsFair) < 0.985\approxFactorClsFair + 1.9851  $.
        \item $ 2+ \sampDF + \approxFactorClsFair(1+\beta+\sampDF) - \dfrac{\alpha\beta}{2(\beta+1)} + \dfrac{2\beta^{2}}{1-\beta} + o_{m}(1) \leq 1.015\approxFactorClsFair + 2 + \sampDF(1+\approxFactorClsFair) - \frac{2\beta^{2}}{1-\beta} + o_{m}(1) \leq 1.015\approxFactorClsFair + 1.9995 $.
        \item $ 1 + 3(\approxFactorClsFair+1)(\approxFactorCorCls+1)\alpha \leq 1 + 0.91(\approxFactorClsFair+1) = 0.91\approxFactorClsFair + 1.091 $.
    \end{itemize}
    Then, $ (1+\errCoreset)(r+\infsmall) \leq 1.0151\approxFactorClsFair + 1.99951 $.

    Under this setting,~\cref{thm:streaming.main.theorem} yields a $ (1.0151\approxFactorClsFair + 1.99951) $-approximation algorithm for the $ k $-median fair consensus clustering problem with space complexity of $ O(k^{2}n\polylog(mn)) $, has update time of $ O((km)^{O(1)}n^{2}\log n ) $, and query time of $ O((k\log(mn))^{O(k)}n^{2} + k^{3}(n+\runtimeClsFair{n})\log^{12} m\log^{3} n) $.
\end{proof}

%% file: improved-parameters.tex
\section{Improved Approximation Guarantee}\label{sec:Improved Approximation} 
In this section, we further optimize the parameters in our framework to achieve an improved approximation ratio for fair consensus clustering in the case of two colors, denoted red and blue. The approximation ratio $ \approxFactorClsFair $ of the closest fair clustering algorithm is chosen from~\cite{chakraborty2025towards}, under three regimes: when the ratio of red to blue points is $1\!:\!1$ ($ \approxFactorClsFair=1 $), $1\!:\!p$ ($ \approxFactorClsFair=17 $), and $p\!:\!q$ ($ \approxFactorClsFair=33 $), where $p$ and $q$ are coprime integers. The improved results are summarized in \cref{tab:optimizing.approx.ratio}.

\begin{table*}[htbp]
    \begin{tabular}{|l|p{1cm}|>{\centering\arraybackslash}p{2.5cm}|>{\centering\arraybackslash}p{2cm}|l|l|l|l|l|l|}
\hline
Problem                                  & Color ratio & Generic Approximation Ratio & Improved Approximation ratio ($ r $) & $ \approxFactorClsFair $ & $ \alpha $ & $ \beta $  & $ c $    & $ s $ & $ g $  \\ \hline
\multirow{3}{*}{$ 1$--median}            & $ 1:1 $     &        $ 2.92 $ (\cref{thm:explicit.fair.consensus.clustering})                    & \textbf{$ 2.901 $}                       & $ 1 $                        &  $ 0.129835 $   &  $ 0.050115 $   &  $ 2.6237 $   &   -   &    -   \\ \cline{2-10} 
                                         & $ 1:p $     &        $ 18.92 $ (\cref{thm:explicit.fair.consensus.clustering})                    &  \textbf{$ 18.896 $}                      &  $ 17 $                       &  $ 0.135970 $   &  $ 0.005765 $   &   $ 2.6161 $   &   -    &    -   \\ \cline{2-10} 
                                         & $ p:q $     &        $ 34.92 $  (\cref{thm:explicit.fair.consensus.clustering})                    &  \textbf{$ 34.905 $}                      &  $ 33 $                       &  $ 0.136393 $   &  $ 0.003407 $   &  $ 2.8742 $   &   -   &    -   \\ \hline
\multirow{3}{*}{streaming $ 1 $--median} & $ 1:1 $     &      $ 2.995 $ (\cref{thm:explicit.fair.consensus.clustering.randomized})                       &  \textbf{$ 2.927 $}                       &  $ 1 $                        &  $ 0.13181506 $ &  $ 0.03626050 $ &  $ 3.270833 $ &  $ 10 $    &  $ 100000 $ \\ \cline{2-10} 
                                         & $ 1:p $     &     $ 18.995 $ (\cref{thm:explicit.fair.consensus.clustering.randomized})                        &  \textbf{$ 18.925 $}                      &  $ 17 $                       &  $ 0.13620638 $ &  $ 0.00415431 $ &  $ 3.270833 $ &  $ 10 $    &  $ 100000 $ \\ \cline{2-10} 
                                         & $ p:q $     &     $ 34.995 $ (\cref{thm:explicit.fair.consensus.clustering.randomized})                        &  \textbf{$ 34.925 $}                      &  $ 33 $                       &  $ 0.13647381 $ &  $ 0.00219897 $ &  $ 3.270833 $ &  $ 10 $    &  $ 100000 $ \\ \hline
\multirow{3}{*}{$ k$--median}            & $ 1:1 $     &        $ 2.92 $ (\cref{thm:explicit.k.fair.consensus.clustering})                    & \textbf{$ 2.901 $}                       & $ 1 $                        &  $ 0.129835 $   &  $ 0.050115 $   &  $ 2.6237 $   &   -   &    -   \\ \cline{2-10} 
                                         & $ 1:p $     &        $ 18.92 $ (\cref{thm:explicit.k.fair.consensus.clustering})                    &  \textbf{$ 18.896 $}                      &  $ 17 $                       &  $ 0.135970 $   &  $ 0.005765 $   &   $ 2.6161 $   &   -   &    -   \\ \cline{2-10} 
                                         & $ p:q $     &        $ 34.92 $  (\cref{thm:explicit.k.fair.consensus.clustering})                    &  \textbf{$ 34.905 $}                      &  $ 33 $                       &  $ 0.136393 $   &  $ 0.003407 $   &  $ 2.8742 $   &   -   &    -   \\ \hline
\end{tabular}
\caption{Improved approximation ratio for fair consensus clustering with two colors achieved by refining parameters}
\label{tab:optimizing.approx.ratio}
\end{table*}


%% file: sample.bib
@inproceedings{chakraborty2023clustering,
  title={Clustering permutations: New techniques with streaming applications},
  author={Chakraborty, Diptarka and Das, Debarati and Krauthgamer, Robert},
  booktitle={14th Innovations in Theoretical Computer Science Conference (ITCS 2023)},
  pages={31--1},
  year={2023},
  organization={Schloss Dagstuhl--Leibniz-Zentrum f{\"u}r Informatik}
}

@inproceedings{cao2025solving,
  title={Solving the Correlation Cluster LP in Sublinear Time},
  author={Cao, Nairen and Cohen-Addad, Vincent and Lee, Euiwoong and Li, Shi and Lolck, David Rasmussen and Newman, Alantha and Thorup, Mikkel and Vogl, Lukas and Yan, Shuyi and Zhang, Hanwen},
  booktitle={Proceedings of the 57th Annual ACM Symposium on Theory of Computing},
  pages={1154--1165},
  year={2025}
}

@inproceedings{bercea2019cost,
  title={On the Cost of Essentially Fair Clusterings},
  author={Bercea, Ioana O and Gro{\ss}, Martin and Khuller, Samir and Kumar, Aounon and R{\"o}sner, Clemens and Schmidt, Daniel R and Schmidt, Melanie},
  booktitle={Approximation, Randomization, and Combinatorial Optimization. Algorithms and Techniques (APPROX/RANDOM 2019)},
  pages={18--1},
  year={2019},
  organization={Schloss Dagstuhl--Leibniz-Zentrum f{\"u}r Informatik}
}

@article{bandyapadhyay2024coresets,
  title={On coresets for fair clustering in metric and euclidean spaces and their applications},
  author={Bandyapadhyay, Sayan and Fomin, Fedor V and Simonov, Kirill},
  journal={Journal of Computer and System Sciences},
  volume={142},
  pages={103506},
  year={2024},
  publisher={Elsevier}
}

@article{bandyapadhyay2024polynomial,
  title={A Polynomial-Time Approximation for Pairwise Fair $ k $-Median Clustering},
  author={Bandyapadhyay, Sayan and Chlamt{\'a}{\v{c}}, Eden and Friggstad, Zachary and Jamshidian, Mahya and Makarychev, Yury and Vakilian, Ali},
  journal={arXiv preprint arXiv:2405.10378},
  year={2024}
}

@inproceedings{bandyapadhyay2025constant,
  title={A Constant-Factor Approximation for Pairwise Fair k-Center Clustering},
  author={Bandyapadhyay, Sayan and Chen, Tianzhi and Friggstad, Zachary and Jamshidian, Mahya},
  booktitle={International Conference on Integer Programming and Combinatorial Optimization},
  pages={43--57},
  year={2025},
  organization={Springer}
}

@inproceedings{chakraborty2025improved,
  title     = {Improved Rank Aggregation Under Fairness Constraint},
  author    = {Chakraborty, Diptarka and Das, Himika and Dey, Sanjana and Yan, Alvin Hong Yao},
  booktitle = {Proceedings of the Thirty-Fourth International Joint Conference on
               Artificial Intelligence, {IJCAI-25}},
  publisher = {International Joint Conferences on Artificial Intelligence Organization},
  editor    = {James Kwok},
  pages     = {330--338},
  year      = {2025},
  month     = {8},
  note      = {Main Track}
}

@inproceedings{chakraborty2025towards,
  author       = {Diptarka Chakraborty and
                  Kushagra Chatterjee and
                  Debarati Das and
                  Tien Long Nguyen and
                  Romina Nobahari},
  editor       = {Nika Haghtalab and
                  Ankur Moitra},
  title        = {Towards Fair Representation: Clustering and Consensus},
  booktitle    = {The Thirty Eighth Annual Conference on Learning Theory, 30-4 July
                  2025, Lyon, France},
  series       = {Proceedings of Machine Learning Research},
  volume       = {291},
  pages        = {838--853},
  publisher    = {{PMLR}},
  year         = {2025}
}

@inproceedings{chakraborty2026,
  author       = {Diptarka Chakraborty and
                  Kushagra Chatterjee and
                  Debarati Das and
                  Tien Long Nguyen},
  title        = {Generalizing Fair Clustering to Multiple Groups: Algorithms and Applications},
  booktitle    = {The Fortieth AAAI Conference on Artificial Intelligence (AAAI-26), Singapore},
  year         = {2026}
}

@inproceedings{feldman2011unified,
  title={A unified framework for approximating and clustering data},
  author={Feldman, Dan and Langberg, Michael},
  booktitle={Proceedings of the forty-third annual ACM symposium on Theory of Computing},
  pages={569--578},
  year={2011}
}

@inproceedings{bachem2018one,
  title={One-shot coresets: The case of k-clustering},
  author={Bachem, Olivier and Lucic, Mario and Lattanzi, Silvio},
  booktitle={International conference on artificial intelligence and statistics},
  pages={784--792},
  year={2018},
  organization={PMLR}
}

@inproceedings{braverman2021coresets,
  title={Coresets for clustering in excluded-minor graphs and beyond},
  author={Braverman, Vladimir and Jiang, Shaofeng H-C and Krauthgamer, Robert and Wu, Xuan},
  booktitle={Proceedings of the 2021 ACM-SIAM Symposium on Discrete Algorithms (SODA)},
  pages={2679--2696},
  year={2021},
  organization={SIAM}
}

@inproceedings{wei22,
  author    = {Dong Wei and
               Md Mouinul Islam and
               Baruch Schieber and
               Senjuti Basu Roy},
  title     = {Rank Aggregation with Proportionate Fairness},
  booktitle = {{SIGMOD} International Conference on Management of Data},
  pages     = {262--275},
  year      = {2022},
}

@inproceedings{chierichetti2017fair,
  title={Fair clustering through fairlets},
  author={Chierichetti, Flavio and Kumar, Ravi and Lattanzi, Silvio and Vassilvitskii, Sergei},
  booktitle={Advances in Neural Information Processing Systems ({NeurIPS})},
  pages={5029-5037},
  year={2017}
}

@inproceedings{BackursIOSVW19,
  author    = {Arturs Backurs and
               Piotr Indyk and
               Krzysztof Onak and
               Baruch Schieber and
               Ali Vakilian and
               Tal Wagner},
  title     = {Scalable Fair Clustering},
  booktitle = {International Conference on Machine Learning (ICML)},
  volume    = {97},
  pages     = {405--413},
  year      = {2019}
}

@inproceedings{BeraCFN19,
  author    = {Suman Kalyan Bera and
               Deeparnab Chakrabarty and
               Nicolas Flores and
               Maryam Negahbani},
  title     = {Fair Algorithms for Clustering},
  booktitle = {Advances in Neural Information Processing Systems (NeurIPS)},
  pages     = {4955--4966},
  year      = {2019}
}

@inproceedings{ChenFLM19,
  author    = {Xingyu Chen and
               Brandon Fain and
               Liang Lyu and
               Kamesh Munagala},
  title     = {Proportionally Fair Clustering},
  booktitle = {International Conference on Machine Learning (ICML)},
  pages     = {1032--1041},
  year      = {2019}
}

@inproceedings{HuangJV19,
  author    = {Lingxiao Huang and
               Shaofeng H.{-}C. Jiang and
               Nisheeth K. Vishnoi},
  title     = {Coresets for Clustering with Fairness Constraints},
  booktitle = {Advances in Neural Information Processing Systems (NeurIPS)},
  pages     = {7587--7598},
  year      = {2019}
}

@inproceedings{dwork2012fairness,
  title={Fairness through awareness},
  author={Dwork, Cynthia and Hardt, Moritz and Pitassi, Toniann and Reingold, Omer and Zemel, Richard},
  booktitle={Innovations in Theoretical Computer Science},
  pages={214--226},
  year={2012}
}

@inproceedings{hardt2016equality,
  title={Equality of opportunity in supervised learning},
  author={Hardt, Moritz and Price, Eric and Srebro, Nati},
  booktitle={Advances in Neural Information Processing Systems (NeurIPS)},
  pages = {3315--3323},
  year={2016}
}

@InProceedings{pmlr-v108-ahmadian20a,
  title = 	 {Fair Correlation Clustering},
  author =       {Ahmadian, Sara and Epasto, Alessandro and Kumar, Ravi and Mahdian, Mohammad},
  booktitle = 	 {International Conference on Artificial Intelligence and Statistics (AISTATS)},
  pages = 	 {4195--4205},
  year = 	 {2020},
  volume = 	 {108}
}

@inproceedings{kay2015unequal,
  title={Unequal representation and gender stereotypes in image search results for occupations},
  author={Kay, Matthew and Matuszek, Cynthia and Munson, Sean A},
  booktitle={{ACM} conference on human factors in computing systems},
  pages={3819--3828},
  year={2015}
}

@article{bolukbasi2016man,
  title={Man is to computer programmer as woman is to homemaker? debiasing word embeddings},
  author={Bolukbasi, Tolga and Chang, Kai-Wei and Zou, James Y and Saligrama, Venkatesh and Kalai, Adam T},
  journal={Advances in Neural Information Processing Systems (NeurIPS)},
  volume={29},
  year={2016}
}

@article{chakraborty2022,
  title={Fair rank aggregation},
  author={Chakraborty, Diptarka and Das, Syamantak and Khan, Arindam and Subramanian, Aditya},
  journal={Advances in Neural Information Processing Systems},
  volume={35},
  pages={23965--23978},
  year={2022},
  note = {Full version: arXiv preprint arXiv:2308.10499}
}

@inproceedings{cao2024understanding,
  title={Understanding the Cluster Linear Program for Correlation Clustering},
  author={Cao, Nairen and Cohen-Addad, Vincent and Lee, Euiwoong and Li, Shi and Newman, Alantha and Vogl, Lukas},
  booktitle={Proceedings of the 56th Annual ACM Symposium on Theory of Computing},
  pages={1605--1616},
  year={2024}
}

@inproceedings{ahmadian2023improved,
  title={Improved approximation for fair correlation clustering},
  author={Ahmadian, Sara and Negahbani, Maryam},
  booktitle={International Conference on Artificial Intelligence and Statistics},
  pages={9499--9516},
  year={2023},
  organization={PMLR}
}

@article{ahmadi2020fair,
  title={Fair correlation clustering},
  author={Ahmadi, Saba and Galhotra, Sainyam and Saha, Barna and Schwartz, Roy},
  journal={arXiv preprint arXiv:2002.03508},
  year={2020}
}

@article{lancichinetti2012consensus,
  title={Consensus clustering in complex networks},
  author={Lancichinetti, Andrea and Fortunato, Santo},
  journal={Scientific reports},
  volume={2},
  number={1},
  pages={336},
  year={2012},
  publisher={Nature Publishing Group UK London}
}

@inproceedings{goder2008consensus,
  title={Consensus clustering algorithms: Comparison and refinement},
  author={Goder, Andrey and Filkov, Vladimir},
  booktitle={2008 Proceedings of the Tenth Workshop on Algorithm Engineering and Experiments (ALENEX)},
  pages={109--117},
  year={2008},
  organization={SIAM}
}

@article{monti2003consensus,
  title={Consensus clustering: a resampling-based method for class discovery and visualization of gene expression microarray data},
  author={Monti, Stefano and Tamayo, Pablo and Mesirov, Jill and Golub, Todd},
  journal={Machine learning},
  volume={52},
  pages={91--118},
  year={2003},
  publisher={Springer}
}

@article{wu2014k,
  title={K-means-based consensus clustering: A unified view},
  author={Wu, Junjie and Liu, Hongfu and Xiong, Hui and Cao, Jie and Chen, Jian},
  journal={IEEE transactions on knowledge and data engineering},
  volume={27},
  number={1},
  pages={155--169},
  year={2014},
  publisher={IEEE}
}

@article{BonizzoniVDJ08,
  author       = {Paola Bonizzoni and
                  Gianluca Della Vedova and
                  Riccardo Dondi and
                  Tao Jiang},
  title        = {On the Approximation of Correlation Clustering and Consensus Clustering},
  journal      = {J. Comput. Syst. Sci.},
  volume       = {74},
  number       = {5},
  pages        = {671--696},
  year         = {2008}
}

@article{filkov2004integrating,
  title={Integrating microarray data by consensus clustering},
  author={Filkov, Vladimir and Skiena, Steven},
  journal={International Journal on Artificial Intelligence Tools},
  volume={13},
  number={04},
  pages={863--880},
  year={2004},
  publisher={World Scientific}
}

@inproceedings{filkov2004heterogeneous,
  title={Heterogeneous data integration with the consensus clustering formalism},
  author={Filkov, Vladimir and Skiena, Steven},
  booktitle={International Workshop on Data Integration in the Life Sciences},
  pages={110--123},
  year={2004},
  organization={Springer}
}

@inproceedings{topchy2003combining,
  title={Combining multiple weak clusterings},
  author={Topchy, Alexander and Jain, Anil K and Punch, William},
  booktitle={Third IEEE international conference on data mining},
  pages={331--338},
  year={2003},
  organization={IEEE}
}

@article{kvrivanek1986np,
  title={NP-hard problems in hierarchical-tree clustering},
  author={K{\v{r}}iv{\'a}nek, Mirko and Mor{\'a}vek, Jaroslav},
  journal={Acta informatica},
  volume={23},
  pages={311--323},
  year={1986},
  publisher={Springer}
}

@inproceedings{swamy2004correlation,
  title={Correlation Clustering: maximizing agreements via semidefinite programming.},
  author={Swamy, Chaitanya},
  booktitle={SODA},
  volume={4},
  pages={526--527},
  year={2004},
  organization={Citeseer}
}

@article{ailon2008aggregating,
  title={Aggregating inconsistent information: ranking and clustering},
  author={Ailon, Nir and Charikar, Moses and Newman, Alantha},
  journal={Journal of the ACM (JACM)},
  volume={55},
  number={5},
  pages={1--27},
  year={2008},
  publisher={ACM New York, NY, USA}
}

@article{braverman2021metric,
  title={Metric k-median clustering in insertion-only streams},
  author={Braverman, Vladimir and Lang, Harry and Levin, Keith and Rudoy, Yevgeniy},
  journal={Discrete Applied Mathematics},
  volume={304},
  pages={164--180},
  year={2021},
  publisher={Elsevier}
}

@book{indyk2001high,
  title={High-dimensional computational geometry},
  author={Indyk, Piotr},
  year={2001},
  publisher={stanford university}
}

@inproceedings{indyk1999sublinear,
  title={Sublinear time algorithms for metric space problems},
  author={Indyk, Piotr},
  booktitle={Proceedings of the thirty-first annual ACM symposium on Theory of computing},
  pages={428--434},
  year={1999}
}

@article{Muthukrishnan05,
  title={Data streams: Algorithms and applications},
  author={Muthukrishnan, Shanmugavelayutham and others},
  journal={Foundations and Trends{\textregistered} in Theoretical Computer Science},
  volume={1},
  number={2},
  pages={117--236},
  year={2005},
  publisher={Now Publishers, Inc.}
}

@inproceedings{BabcockEtAl02,
  title={Models and issues in data stream systems},
  author={Babcock, Brian and Babu, Shivnath and Datar, Mayur and Motwani, Rajeev and Widom, Jennifer},
  booktitle={Proceedings of the twenty-first ACM SIGMOD-SIGACT-SIGART symposium on Principles of database systems},
  pages={1--16},
  year={2002}
}

@inproceedings{HarPeledMazumdar04,
  author    = {Sariel Har{-}Peled and Soham Mazumdar},
  title     = {On coresets for k-means and k-median clustering and their applications to approximate nearest neighbor and clustering in data streams},
  booktitle = {Proceedings of the 36th Annual ACM Symposium on Theory of Computing (STOC)},
  pages     = {291--300},
  year      = {2004}
}

@article{bentley1980decomposable,
  title={Decomposable searching problems I. Static-to-dynamic transformation},
  author={Bentley, Jon Louis and Saxe, James B},
  journal={Journal of Algorithms},
  volume={1},
  number={4},
  pages={301--358},
  year={1980},
  publisher={Elsevier}
}

@inproceedings{FeldmanLangberg11,
  author    = {Dan Feldman and Michael Langberg},
  title     = {A unified framework for approximating and clustering data},
  booktitle = {Proceedings of the 43rd Annual ACM Symposium on Theory of Computing (STOC)},
  pages     = {569--578},
  year      = {2011}
}

@inproceedings{CharikarOCallaghanPanigrahy03,
  author    = {Moses Charikar and Liadan O'Callaghan and Rina Panigrahy},
  title     = {Better streaming algorithms for clustering problems},
  booktitle = {Proceedings of the 35th Annual ACM Symposium on Theory of Computing (STOC)},
  pages     = {30--39},
  year      = {2003}
}

@inproceedings{schmidt2019fair,
  title={Fair coresets and streaming algorithms for fair k-means},
  author={Schmidt, Melanie and Schwiegelshohn, Chris and Sohler, Christian},
  booktitle={International Workshop on Approximation and Online Algorithms},
  pages={232--251},
  year={2019},
  organization={Springer}
}

@article{braverman2019streaming,
  title={Streaming coreset constructions for m-estimators},
  author={Braverman, Vladimir and Feldman, Dan and Lang, Harry and Rus, Daniela},
  journal={Approximation, Randomization, and Combinatorial Optimization. Algorithms and Techniques},
  year={2019}
}

@article{rosman2014coresets,
  title={Coresets for k-segmentation of streaming data},
  author={Rosman, Guy and Volkov, Mikhail and Feldman, Danny and Fisher, John W and Rus, Daniela},
  journal={Advances in neural information processing systems},
  volume={27},
  year={2014}
}

@article{huang2019coresets,
  title={Coresets for clustering with fairness constraints},
  author={Huang, Lingxiao and Jiang, Shaofeng and Vishnoi, Nisheeth},
  journal={Advances in neural information processing systems},
  volume={32},
  year={2019}
}

@inproceedings{chhaya2022coresets,
  title={On Coresets for Fair Regression and Individually Fair Clustering.},
  author={Chhaya, Rachit and Dasgupta, Anirban and Choudhari, Jayesh and Shit, Supratim},
  booktitle={AISTATS},
  pages={9603--9625},
  year={2022}
}

@inproceedings{braverman2022power,
  title={The power of uniform sampling for coresets},
  author={Braverman, Vladimir and Cohen-Addad, Vincent and Jiang, H-C Shaofeng and Krauthgamer, Robert and Schwiegelshohn, Chris and Toftrup, Mads Bech and Wu, Xuan},
  booktitle={2022 IEEE 63rd Annual Symposium on Foundations of Computer Science (FOCS)},
  pages={462--473},
  year={2022},
  organization={IEEE}
}

@article{xiong2024fair,
  title={Fair wasserstein coresets},
  author={Xiong, Zikai and Dalmasso, Niccol{\`o} and Sharma, Shubham and Lecue, Freddy and Magazzeni, Daniele and Potluru, Vamsi and Balch, Tucker and Veloso, Manuela},
  journal={Advances in Neural Information Processing Systems},
  volume={37},
  pages={132--168},
  year={2024}
}

@inproceedings{caruana2006meta,
  title={Meta clustering},
  author={Caruana, Rich and Elhawary, Mohamed and Nguyen, Nam and Smith, Casey},
  booktitle={Sixth International Conference on Data Mining (ICDM'06)},
  pages={107--118},
  year={2006},
  organization={IEEE}
}

@article{topchy2005clustering,
  title={Clustering ensembles: Models of consensus and weak partitions},
  author={Topchy, Alexander and Jain, Anil K and Punch, William},
  journal={IEEE transactions on pattern analysis and machine intelligence},
  volume={27},
  number={12},
  pages={1866--1881},
  year={2005},
  publisher={IEEE}
}

@article{bellec2010multi,
  title={Multi-level bootstrap analysis of stable clusters in resting-state fMRI},
  author={Bellec, Pierre and Rosa-Neto, Pedro and Lyttelton, Oliver C and Benali, Habib and Evans, Alan C},
  journal={Neuroimage},
  volume={51},
  number={3},
  pages={1126--1139},
  year={2010},
  publisher={Elsevier}
}
